\def\year{2020}\relax
\documentclass[letterpaper]{article} 
\usepackage{aaai20}  
\usepackage{times}  
\usepackage{helvet} 
\usepackage{courier}  
\usepackage[hyphens]{url}  
\usepackage{graphicx} 
\usepackage{graphicx}  
\usepackage{xcolor}
\usepackage{amsmath}
\usepackage{amsthm}
\usepackage{amssymb}
\theoremstyle{definition}
\newtheorem{definition}{Definition}
\newtheorem{theorem}{Theorem}
\newtheorem{corollary}{Corollary}
\newtheorem{lemma}{Lemma}
\newcommand{\del}[1]{}
\newcommand\norm[1]{\lVert#1\rVert}

\usepackage[linesnumbered, ruled, vlined]{algorithm2e}

\SetCommentSty{mycommfont}

\usepackage[caption=false]{subfig}
\usepackage{stmaryrd}
\usepackage{multirow}
\newcommand{\citet}[1]{\citeauthor{#1} \shortcite{#1}}
\newcommand{\mb}[1]{\mathbf{#1}}

\newcommand{\mc}[1]{\mathcal{#1}}
\newcommand{\mbb}[1]{\mathbb{#1}}

\title{Privacy-Preserving Gradient Boosting Decision Trees}
\author{
Qinbin Li,\textsuperscript{\rm 1}
Zhaomin Wu,\textsuperscript{\rm 1}
Zeyi Wen,\textsuperscript{\rm 2}
Bingsheng He\textsuperscript{\rm 1}\\
\textsuperscript{\rm 1}National University of Singapore\\
\textsuperscript{\rm 2}The University of Western Australia\\
\{qinbin, zhaomin, hebs\}@comp.nus.edu.sg, zeyi.wen@uwa.edu.au
}

\begin{document}

\maketitle

\begin{abstract}
The Gradient Boosting Decision Tree (GBDT) is a popular machine learning model for various tasks in recent years. In this paper, we study how to improve model accuracy of GBDT while preserving the strong guarantee of differential privacy. \textit{Sensitivity} and \textit{privacy budget} are two key design aspects for the effectiveness of differential private models. Existing solutions for GBDT with differential privacy suffer from the significant accuracy loss due to too loose sensitivity bounds and ineffective privacy budget allocations (especially across different trees in the GBDT model). Loose sensitivity bounds lead to more noise to obtain a fixed privacy level. Ineffective privacy budget allocations worsen the accuracy loss especially when the number of trees is large. Therefore, we propose a new GBDT training algorithm that achieves tighter sensitivity bounds and more effective noise allocations. Specifically, by investigating the property of gradient and the contribution of each tree in GBDTs, we propose to adaptively control the gradients of training data for each iteration and leaf node clipping in order to tighten the sensitivity bounds. Furthermore, we design a novel boosting framework to allocate the privacy budget between trees so that the accuracy loss can be further reduced. Our experiments show that our approach can achieve much better model accuracy than other baselines.
\end{abstract}

\section{Introduction}

Gradient Boosting Decision Trees (GBDTs) have achieved state-of-the-art results on many challenging machine learning tasks such as click prediction~\cite{richardson2007predicting}, learning to rank~\cite{burges2010ranknet}, and web page classification~\cite{pennacchiotti2011machine}. The algorithm builds a number of decision trees one by one, where each tree tries to fit the residual of the previous trees. With the development of efficient GBDT libraries~\cite{chen2016xgboost,ke2017lightgbm,prokhorenkova2018catboost,wenthundergbm19}, the GBDT model has won many awards in recent machine learning competitions and has been widely used both in the academics and in the industry~\cite{he2014practical,zhou2017psmart,ke2017lightgbm,jiang2018dimboost,feng2018multi}. 

Privacy issues have been a hot research topic recently~\cite{shokri2017membership,truex2018towards,fredrikson2015model,li2019federated,li2019practical}. Due to the popularity and wide adoptions of GBDTs, a privacy-preserving GBDT algorithm is particularly timely and necessary. Differential privacy~\cite{dwork2011differential} was proposed to protect the individuals of a dataset. In short, a computation is differentially private if the probability of producing a given output does not depend much on whether a particular record is included in the input dataset. Differential privacy has been widely used in many machine learning models such as logistic regression~\cite{chaudhuri2009privacy} and neural networks~\cite{abadi2016deep,acs2018differentially}. \textit{Sensitivity} and \textit{privacy budget} are two key design aspects for the effectiveness of differential private models. Many practical differentially private models achieve good model quality by deriving tight sensitivity bounds and allocating privacy budget effectively. In this paper, we study how to improve model accuracy of GBDTs while preserving the strong guarantee of differential privacy.  

There have been some potential solutions for improving the effectiveness of differentially private GBDTs (e.g.,~\cite{zhao2018inprivate,liu2018differentially,xiang2018collaborative}).  However, they can suffer from the significant accuracy loss due to too loose sensitivity bounds and ineffective privacy budget allocations (especially across different trees in the GBDT model).

\emph{Sensitivity bounds:} The previous studies on individual decision trees~\cite{friedman2010data,mohammed2011differentially,liu2018differentially} bound the sensitivities by estimating the range of the function output. However, this method leads to very loose sensitivity bounds in GBDTs, because the range of the gain function output ($G$ in Equation~\eqref{eq:gain} introduced Section~\ref{sec:background}) is related to the number of instances and the range can be potentially very huge for large data sets. Loose sensitivity bounds lead to more noise to obtain a fixed privacy level, and cause huge accuracy loss. 

\emph{Privacy budget allocations:} There have been some previous studies on privacy budget allocations among different trees~\cite{liu2018differentially,xiang2018collaborative,zhao2018inprivate}. We can basically divide them into two kinds.  1) The first kind is to allocate the budget equally to each tree using the \textit{sequential composition}~\cite{liu2018differentially,xiang2018collaborative}. When the number of trees is large, the given budget allocated to each tree is very small. The scale of the noises can be proportional to the number of trees, which causes huge accuracy loss. 2) The second kind is to give disjoint inputs to different trees~\cite{zhao2018inprivate}. Then, each tree only needs to satisfy $\varepsilon$-differential privacy using the \textit{parallel composition}. When the number of trees is large, since the inputs to the trees cannot be overlapped, the number of instances assigned to a tree can be quite small. As a result, the tree can be too weak to achieve meaningful learnt models.

We design a new GBDT training algorithm to address the above-mentioned limitations.
\begin{itemize}
    \item In order to obtain a tighter sensitivity bound, we propose Gradient-based Data Filtering (GDF) to guarantee the bounds of the sensitivities, and further propose Geometric Leaf Clipping (GLC) to obtain a closer bound on sensitivities taking advantage of the tree learning systems in GBDT.
    \item Combining both sequential and parallel compositions, we design a novel boosting framework to well exploit the privacy budget of GBDTs and the effect of boosting. Our approach satisfies differential privacy while improving the model accuracy with boosting. 
    \item We have implemented our approach (named DPBoost) based on a popular library called LightGBM~\cite{ke2017lightgbm}. Our experimental results show that DPBoost is much superior to the other approaches and can achieve competitive performance compared with the ordinary LightGBM.
\end{itemize}

\section{Preliminaries}\label{sec:background}

\subsection{Gradient Boosting Decision Trees}

The GBDT is an ensemble model which trains a number of decision trees in a sequential manner. Formally, given a convex loss function $l$ and a dataset with $n$ instances and $d$ features $\{(\mathbf{x}_i, y_i)\}_{i=1}^n (\mathbf{x}_i\in \mathbb{R}^d, y_i\in \mathbb{R})$, 
GBDT minimizes the following regularized objective~\cite{chen2016xgboost}.
\begin{equation}
    \mathcal{\tilde{L}} = \sum_il(\hat{y}_i,y_i)+\sum_k \Omega(f_k) 
\end{equation}
where $\Omega(f)= \frac{1}{2}\lambda \norm{V}^2$ is a regularization term. Here $\lambda$ is the regularization parameter and $V$ is the leaf weight. Each $f_k$ corresponds to a decision tree.
Forming an approximate function of the loss, GBDT minimizes the following objective function at the $t$-th iteration~\cite{si2017gradient}.
\begin{equation}
    \mathcal{\tilde{L}}^{(t)} = \sum_{i=1}^{n} [g_i f_t(\mathbf{x}_i)+\frac{1}{2} f_t^2(\mathbf{x}_i)]+\Omega(f_t)
\label{eq:xg_loss}
\end{equation}

where $g_i = \partial_{\hat{y}^{(t-1)}}l(y_i,\hat{y}^{(t-1)})$ is first order gradient statistics on the loss function. The decision tree is built from the root until reaching the maximum depth. Assume $I_L$ and $I_R$ are the instance sets of left and right nodes after a split. Letting $I = I_L \cup I_R$, the gain of the split is given by

\begin{equation}
    G(I_L, I_R) = \frac{(\sum_{i\in I_L}g_i)^2}{|I_L|+\lambda}+\frac{(\sum_{i\in I_R}g_i)^2}{|I_R|+\lambda}
\label{eq:gain}
\end{equation}

GBDT traverses all the feature values to find the split that maximizes the gain. If the current node does not meet the requirements of splitting (e.g., achieve the max depth or the gain is smaller than zero), it becomes a leaf node and the optimal leaf value is given by

\begin{equation}
    V(I) = -\frac{\sum_{i \in I}g_i}{|I| + \lambda}
\label{eq:leaf}
\end{equation}

Like the learning rate in stochastic optimization, a shrinkage rate $\eta$~\cite{friedman2002stochastic} is usually applied to the leaf values, which can reduce the influence of each individual tree and leave space for future trees to improve the model.
\subsection{Differential Privacy}
Differential privacy~\cite{dwork2011differential} is a popular standard of privacy protection with provable privacy guarantee. It guarantees that the probability of producing a given output does not depend much on whether a particular record is included in the input dataset or not.

\begin{definition} ($\varepsilon$-Differential Privacy)
Let $\varepsilon$ be a positive real number and $f$ be a randomized function. The function $f$ is said to provide $\varepsilon$-differential privacy if, for any two datasets $D$ and $D'$ that differ in a single record and any output $O$ of function $f$,
\begin{equation}
    Pr[f(D)\in O]\leq e^\varepsilon \cdot Pr[f(D')\in O]
\end{equation}
\end{definition}

Here $\varepsilon$ is a privacy budget. To achieve $\varepsilon$-differential privacy, the Laplace mechanism and exponential mechanism~\cite{dwork2014algorithmic} are usually adopted by adding noise calibrated to the \textit{sensitivity} of a function. 

\begin{definition} (Sensitivity)
Let $f:\mathcal{D} \rightarrow \mathcal{R}^d$ be a function. The sensitivity of $f$ is
		\begin{equation}
		    \Delta f = \max_{D,D'\in \mathcal{D}}\norm{f(D)-f(D')}_1
		\end{equation}
where $D$ and $D'$ have at most one different record. 
\label{def:sen}
\end{definition}

\begin{theorem} (Laplace Mechanism)
Let $f:\mathcal{D} \rightarrow \mathcal{R}^d$ be a function. The Laplace Mechanism $F$ is defined as
		\begin{equation}
		    F(D)=f(D)+Lap(0, \Delta f/\varepsilon)
		\end{equation}
		where the noise $Lap(0, \Delta f/\varepsilon)$ is drawn from a Laplace distribution with mean zero and scale $\Delta f/\epsilon$. Then $F$ provides $\epsilon$-differential privacy.
\end{theorem}

\begin{theorem} (Exponential Mechanism)
Let $u:(\mc{D}\times \mc{R}) \rightarrow \mbb{R}$ be a utility function. The exponential mechanism $F$ is defined as

\begin{equation}
\begin{aligned}
    F(D,u) = &\text{ choose } r \in \mc{R} \text{ with}\\
    &\text{probability }\propto exp(\frac{\varepsilon u(D,r)}{2\Delta u})
\end{aligned}
\end{equation}
Then $F$ provides $\varepsilon$-differential privacy.
\end{theorem}

The above-mentioned mechanisms provide privacy guarantees for a single function. For an algorithm with multiple functions, there are two privacy budget composition theorems~\cite{dwork2014algorithmic}.

\begin{theorem} (Sequential Composition)
Let $f=\{f_1, ..., f_m\}$ be a series of functions performed sequentially on a dataset. If $f_i$ provides $\varepsilon_i$-differential privacy, then $f$ provides $\sum_{i=1}^m \varepsilon_i$-differential privacy.
\end{theorem}
\begin{theorem} (Parallel Composition)
Let $f=\{f_1, ..., f_m\}$ be a series of functions performed separately on disjoint subsets of the entire dataset. If $f_i$ provides $\varepsilon_i$-differential privacy, then $f$ provides $\max(\varepsilon_1, ..., \varepsilon_m)$-differential privacy.
\end{theorem}

\section{Our Design: DPBoost}

Given a privacy budget $\varepsilon$ and a dataset with $n$ instances and $d$ features $\mathcal{D}=\{(\mathbf{x}_i, y_i)\}_{i=1}^n (\mathbf{x}_i\in \mathbb{R}^d, y_i\in [-1,1])$, we develop a new GBDT training algorithm named DPBoost to achieve  $\varepsilon$-differential privacy while trying to reduce the accuracy loss. Moreover, like the setting in the previous work~\cite{abadi2016deep}, we consider a strong adversary with full access to the model's parameters. Thus, we also provide differential privacy guarantees for each tree node.

Figure~\ref{fig:framwork} shows the overall framework of DPBoost. We design a novel two-level boosting framework to exploit both sequential composition and parallel composition. Inside an ensemble, a number of trees are trained using the disjoint subsets of data sampled from the dataset. Then, multiple such rounds are trained in a sequential manner. For achieving differential privacy, sequential composition and parallel composition are applied between ensembles and inside an ensemble, respectively. Next, we describe our algorithm in detail, including the techniques to bound sensitivities and effective privacy budget allocations.

\begin{figure}[!t]
\centering
\includegraphics[width=0.95\columnwidth]{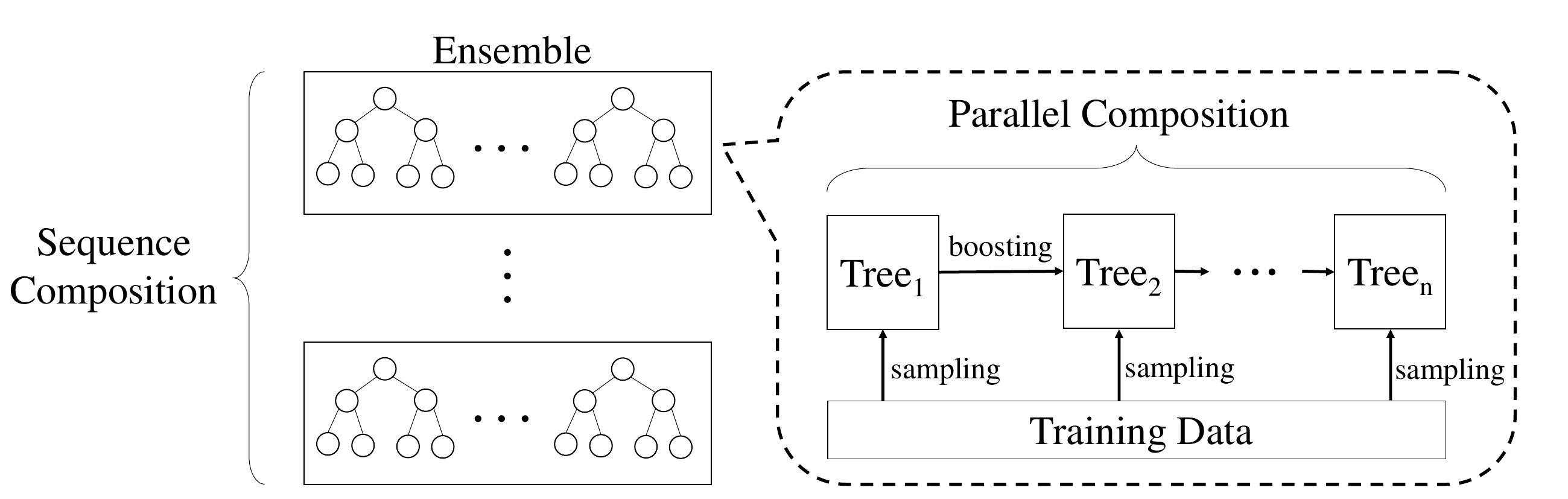}%
\caption{The two-level boosting design of DPBoost}
\label{fig:framwork}
\end{figure}

\subsection{Tighter Sensitivity Bounds} \label{sec:gdf}
The previous studies on individual decision trees~\cite{friedman2010data,mohammed2011differentially,liu2018differentially} bound the sensitivities by estimating the range of the function output. For example, if the range of a function output is $[-\gamma, \gamma]$, then the sensitivity of the function is no more than $2\gamma$. However, the range of function $G$ in Equation~\eqref{eq:gain} is related to the number of instances, which can cause very large sensitivity if the dataset is large. Thus, instead of estimating the range, we strictly derive the sensitivity ($\Delta G$ and $\Delta V$) according to Definition~\ref{def:sen}. Their bounds are given in the below two lemmas.

\begin{lemma}
Letting $g^*=max_{i\in \mc{D}} |g_i|$, we have 
$\Delta G \leq 3{g^*}^2.$
\label{lem:sen_gain}
\end{lemma}

\begin{proof}
Consider two adjacent instance sets $I_1 = \{\mb{x}_i\}_{i=1}^{n}$ and $I_2 = I_1 \cup \{\mb{x}_{s}\}$ that differ in a single instance. Assume $I_1 = I_L \cup I_R$, where $I_L$ and $I_R$ are the instance sets of leaf and right nodes respectively after a split. Without loss of generality, we assume that instance $x_{s}$ belongs to the left node. We use $n_l$ to denote $|I_L|$. Then, we have
\begin{equation}
\begin{aligned}
    &\Delta G = |\frac{(\sum_{i\in I_L}g_i + g_{s})^2}{n_l+\lambda + 1} - \frac{(\sum_{i\in I_L}g_i)^2}{n_l+\lambda}| \\
    &= |\frac{(n_l+\lambda)g_{s}^2+2(n_l+\lambda)g_{s}\sum_{i\in I_L}g_i-(\sum_{i\in I_L}g_i)^2}{(n_l + \lambda + 1)(n_l + \lambda)}|
\end{aligned}
\label{eq:sen_g}
\end{equation}
Let $h(g_s,\sum_{i\in I_L}g_i) = (n_l+\lambda)g_s^2+2(n_l+\lambda)g_s\sum_{i\in I_L}g_i-(\sum_{i\in I_L}g_i)^2$. Let $\partial_{g_s}h = 0$ and $\partial_{\sum_{i\in I_L}g_i}h = 0$, we can get $g_s=0$ and $\sum_{i\in I_L}g_i = 0$. Comparing the stationary point and the boundary (i.e., $g_s = \pm g^*$ and $\sum_{i\in I_L}g_i = \pm n_lg^*$), we can find that when $g_s = -g^*$, $\sum_{i\in I_L}g_i = n_lg^*$ (or $g_s = g^*$ and $\sum_{i\in I_L}g_i = -n_lg^*$), and $n_l \rightarrow \infty$, Equation \eqref{eq:sen_g} can achieve maximum. We have
\begin{equation}
\begin{aligned}
    \Delta G &= |\frac{(n_l-1)^2 {g^*}^2}{n_l+\lambda+1} - \frac{n_l^2 {g^*}^2}{n_l+\lambda}| \\
    &= |\frac{-3n_l^2+(1-2\lambda)n_l+\lambda}{n_l^2+(2\lambda+1)n_l+\lambda(\lambda+1)}|{g^*}^2 \\
    &\leq 3{g^*}^2
\end{aligned}
\end{equation}
\end{proof}

\begin{lemma}
Letting $g^*=max_{i\in \mc{D}} |g_i|$, we have $\Delta V \leq \frac{g^*}{1+\lambda}$.
\label{lem:sen_leaf}
\end{lemma}

\begin{proof}
The proof follows a similar way with the proof of Lemma~\ref{lem:sen_gain}. The detailed proof is available in Appendix A.
\end{proof}

For ease of presentation, we call the absolute value of the gradient as 1-norm gradient. As we can see from Lemma~\ref{lem:sen_gain} and Lemma~\ref{lem:sen_leaf}, the sensitivities of nodes are related to the maximum 1-norm gradient. Since there is no a priori bound on the value of the gradients, we have to restrict the range of gradients. A potential solution is to clip the gradient by a threshold, which is often adopted in deep learning~\cite{abadi2016deep,shokri2015privacy}. However, in GBDTs, since the gradients are computed based on distance between the prediction value and the target value, clipping the gradients means indirectly changing the target value, which may lead to a huge accuracy loss. Here, we propose a new approach named gradient-based data filtering. The basic idea is to restrict the maximum 1-norm gradient by only filtering a very small fraction of the training dataset in each iteration.

\paragraph{Gradient-based Data Filtering (GDF)} At the beginning of the training, the gradient of instance $\mb{x}_i$ is initialized as $g_i = \frac{\partial l(y_i,y)}{\partial y}|_{y=0}$. We let $g_l^* = max_{y_p \in [-1,1]} \lVert \frac{\partial l(y_p,y)}{\partial y}|_{y=0} \rVert$, which is the maximum possible 1-norm gradient in the initialization. Note that $g_l^*$ is \textbf{independent to training data} and only depends on the loss function $l$ (e.g., $g_l^* = 1$ for square loss function).\del{ Then in the initialization before the training of GBDTs, we have $\forall i \in [1,n]$, $|g_i| \leq g_l^*$. Since the loss function $l$ is convex, the 1-norm gradient of an instance tends to decrease as the number of trees grows.} Since the loss function $l$ is convex (i.e., the gradient is monotonically non-decreasing), the values of the 1-norm gradients tend to decrease as the number of trees increases in the training. Consequently, as we have shown the experimental results in Appendix B, most instances have a lower 1-norm gradient than $g_l^*$ during the whole training process. Thus, we can filter the training instances by the threshold $g_l^*$. Specifically, at the beginning of each iteration, we filter the instances that have 1-norm gradient larger than $g_l^*$ (i.e., those instances are not considered in this iteration). Only the remaining instances are used as the input to build a new differentially private decision tree in this iteration. Note that the filtered instances may still participate in the training of the later trees. With such gradient-based data filtering technique, we can ensure that the gradients of the used instances are no larger than $g_l^*$. Then, according to Lemma~\ref{lem:sen_gain} and Lemma~\ref{lem:sen_leaf}, we can bound the sensitivities of $G$ and $V$ as shown in Corollary~\ref{theo:gdf}.

\begin{corollary}
By applying GDF in the training of GBDTs, we have $\Delta G \leq 3{g_l^*}^2$ and $\Delta V \leq \frac{g_l^*}{1+\lambda}$.
\label{theo:gdf}
\end{corollary}

In the following, we analyze the approximation error of GDF.
\begin{theorem}
Given an instance set $I$, suppose $I = I_f \cup I_f^c $, where $I_f$ is the filtered instance set and $I_f^c$ is the remaining instance set in GDF. Let $p = \frac{|I_f|}{|I|}$ and $\bar{g}_f = \frac{\sum_{i\in I_f}g_i}{|I_f|}$. We denote the approximation error of GDF on leaf values as $\xi_I = |V(I) - V(I_f^c)|$. Then, we have $\xi_I \leq p(|\bar{g}_f| + g_l^*)$.
\label{theo:error_GBF}
\end{theorem}

\begin{proof}
With Equation~\eqref{eq:leaf}, we have
\begin{equation}
\begin{aligned}
    \xi_I &= |\frac{\sum_{i\in I_f}g_i + \sum_{i\in I_f^c}g_i}{|I| + \lambda} - \frac{\sum_{i\in I_f^c}g_i}{(1-p)|I| + \lambda}| \\
    &\leq |\frac{\sum_{i\in I_f}g_i}{|I| + \lambda}| + |(\frac{1}{|I|+\lambda} - \frac{1}{(1-p)|I| + \lambda})\sum_{i\in I_f^c}g_i| \\
    &= p|\frac{\sum_{i\in I_f}g_i}{|I_f| + p\lambda}| + p | \frac{|I|}{(|I|+\lambda)((1-p)|I|+\lambda)}\sum_{i\in I_f^c}g_i| \\
    &\leq p|\frac{\sum_{i\in I_f}g_i}{|I_f|}| + p|\frac{\sum_{i\in I_f^c}g_i}{(1-p)|I|}| \leq p(|\bar{g}_f| + g_l^*)
\end{aligned}
\end{equation}
\end{proof}

According to Theorem~\ref{theo:error_GBF}, we have the following discussions: (1) The upper bound of the approximation error of GDF does not depend on the number of instances.  This good property allows small approximation errors even on large data sets. (2) Normally, most instances have gradient values lower than the threshold $g_l^*$ and the ratio $p$ is low, as also shown in Appendix B. Then, the approximation error is small in practice. (3) The approximation error may be large if $|\bar{g}_f|$ is big. However, the instances with a very large gradient are often outliers in the training data set since they cannot be well learned by GBDTs. Thus, it is reasonable to learn a tree by filtering those outliers.

\paragraph{Geometric Leaf Clipping (GLC)} GDF provides the same sensitivities for all trees. Since the gradients tend to decrease from iteration to iteration in the training process, there is an opportunity to derive a tighter sensitivity bound as the iterations go. However, it is too complicated to derive the exact decreasing pattern of the gradients in practice. Also, as discussed in the previous section, gradient clipping with an inappropriate decaying threshold can lead to huge accuracy loss. We need a new approach for controlling this decaying effect across different tree learning. Note, while the noises injected in the internal nodes influence the gain of the current split, the noises injected on the leaf value directly influence the prediction value. Here we focus on bounding the sensitivity of leaf nodes.

Fortunately, according to Equation~\eqref{eq:leaf}, the leaf values also decrease as the gradients decrease. Since the GBDT model trains a tree at a time to fit the residual of the trees that precede it, clipping the leaf nodes would mostly influence the convergence rate but not the objective of GBDTs. Thus, we propose adaptive leaf clipping to achieve a decaying sensitivity on the leaf nodes. Since it is unpractical to derive the exact decreasing pattern of the leaf values in GBDTs, we start with a simple case and further analyze its findings in practice.

\begin{theorem}
Consider a simple case that each leaf has only one single instance during the GBDT training. Suppose the shrinkage rate is $\eta$. We use $V_t$ to denote the leaf value of the $t$-th tree in GBDT. Then, we have $|V_t| \leq g_l^*(1-\eta)^{t-1}$.
\label{theo:gc}
\end{theorem}

\begin{proof}
For simplicity, we assume the label of the instance is -1 and the gradient of the instance is initialized as $g_l^*$. For the first tree, we have $V_1 = -\frac{g_l^*}{1 + \lambda} \geq -g_l^*$. Since the shrinkage rate is $\eta$, the improvement of the prediction value on the first tree is $-\eta g_l^*$. Thus, we have
\begin{equation}
\begin{aligned}
    |V_2| &\leq  \norm{\frac{\partial l(-1,y)}{\partial y}|_{y= \eta V_1}}\\
    &\approx\norm{ \frac{\partial l(-1,y)}{\partial y}|_{y=0} + \eta V_1 } \\
    &\leq  g_l^* (1-\eta)
\end{aligned}
\end{equation}
In the same way, we can get $|V_t| \leq g_l^*(1-\eta)^{t-1}$.
\end{proof}

Although the simple case in Theorem~\ref{theo:gc} may not fully reflect decaying patterns of the leaf value in practice, it can give an insight on the reduction of the leaf values as the number of trees increases. The leaf values in each tree form a geometric sequence with base $g_l^*$ and common ratio $(1-\eta)$. Based on this observation, we propose geometric leaf clipping. Specifically, in the training of the tree in iteration $t$ in GBDTs, we clip the leaf values with the threshold $g_l^*(1-\eta)^{t-1}$ before applying Laplace mechanism (i.e., $V_t' = V_t \cdot \min(1, g_l^*(1-\eta)^{t-1} / |V_t|)$). That means, if the leaf value is larger than the threshold, its value is set to be the threshold. Then, combining with Corollary~\ref{theo:gdf}, we get the following result on bounding the sensitivity on each tree in the training process.

\begin{corollary}
With GDF and GLC, the sensitivity of leaf nodes in the tree of the $t$-th iteration satisfies $\Delta V_t \leq \min(\frac{g_l^*}{1+\lambda}, 2g_l^*(1-\eta)^{t-1})$.
\end{corollary}

We have conducted experiments on the effect of geometric clipping, which are shown in Appendix B. Our experiments show that GLC can effectively improve the performance of DPBoost.

\subsection{Privacy Budget Allocations}
As in Introduction, previous approaches~\cite{liu2018differentially,xiang2018collaborative,zhao2018inprivate} suffer from accuracy loss, due to the ineffective privacy budget allocations across trees. The accuracy loss can be even bigger, when the number of trees in GBDT is large. For completeness, we first briefly present the mechanism for building a single tree $t$ with a given privacy budget $\varepsilon_t$, by using an approach in the previous study~\cite{mohammed2011differentially,zhao2018inprivate}. Next, we present our proposed approach for budget allocation across trees in details. 

\paragraph{Budget Allocation for A Single Tree}  Algorithm~\ref{alg:priv_tree} shows the procedure of learning a differentially private decision tree. In the beginning, we use GDF (introduced in Section~\ref{sec:gdf}) to preprocess the training dataset. Then, the decision tree is built from root until reaching the maximum depth. For the internal nodes, we adopt the exponential mechanism when selecting the split value. Considering the gain $G$ as the utility function, the feature value with higher gain has a higher probability to be chosen as the split value. For the leaf nodes, we first clip the leaf values using GLC (introduced in Section~\ref{sec:gdf}). Then, the Laplace mechanism is applied to inject random noises to the leaf values. For the privacy budget allocation inside a tree, we adopt the mechanism in the existing studies~\cite{mohammed2011differentially,zhao2018inprivate}. Specifically, we allocate a half of the privacy budget for the leaf nodes (i.e., $\varepsilon_\mathit{leaf}$), and then equally divide the remaining budget to each depth of the internal nodes (each level gets $\varepsilon_\mathit{nleaf}$).

\begin{theorem}
The output of Algorithm~\ref{alg:priv_tree} satisfies $\varepsilon_t$-differential privacy.
\end{theorem}

\begin{proof}
Since the nodes in one depth have disjoint inputs, according to the parallel composition, the privacy budget consumption in one depth only need to be counted once. Thus, the total privacy budget consumption is no more than $\varepsilon_\mathit{nleaf} * \mathit{Depth}_\mathit{max} + \varepsilon_\mathit{leaf} = \varepsilon_t$.
\end{proof}

\begin{algorithm}
\SetNoFillComment
\DontPrintSemicolon
\KwIn{$I$: training data, $\mathit{Depth}_\mathit{max}$: maximum depth}
\KwIn{$\varepsilon_t$: privacy budget}
$\varepsilon_\mathit{leaf} \leftarrow \frac{\varepsilon_t}{2}$ \tcp*{privacy budget for leaf nodes}
$\varepsilon_\mathit{nleaf} \leftarrow \frac{\varepsilon_t}{2\mathit{Depth}_\mathit{max}}$ \tcp*{privacy budget for internal nodes}
Perform gradient-based data filtering on dataset $I$.\\
\For {$depth = 1$ \KwTo $\mathit{Depth}_\mathit{max}$}{
    \For{each node in current depth}{
        \For{each split value $i$}{
            Compute gain $G_i$ according to Equation~\eqref{eq:gain}.\\
            $P_i \leftarrow exp(\frac{\varepsilon_\mathit{nleaf} G_i}{2\Delta G})$\\
        }
        \tcc{Apply exponential mechanism}
        Choose a value $s$ with probability $(P_s /\sum_i P_i)$.\\
        Split current node by feature value $s$.
    }
}
\For{each leaf node $i$}{
    Compute leaf value $V_i$ according to Equation~\eqref{eq:leaf}.\\
    Perform geometric leaf clipping on $V_i$.\\
    \tcc{Apply Laplace mechanism}
    $V_i \leftarrow V_i + Lap(0, \Delta V / \varepsilon_\mathit{leaf})$
}
\KwOut{A $\varepsilon_t$-differentially private decision tree}
\caption{TrainSingleTree: Train a differentially private decision tree}
\label{alg:priv_tree}
\end{algorithm}

\paragraph{Budget Allocation Across trees} We propose a two-level boosting structure called Ensemble of Ensembles (EoE), which can exploit both sequential composition and parallel composition to allocate the privacy budget between trees. Within each ensemble, we first train a number of trees with disjoint subsets sampled from the dataset $\mc{D}$. Thus, the parallel composition is applied inside an ensemble. Then, multiple such ensembles are trained in a sequential manner using the same training set $\mc{D}$. As a result, the sequential composition is applied between ensembles. Such a design can utilize the privacy budget while maintaining the effectiveness of boosting. EoE can effectively address the side effect of geometric leaf clipping in some cases, which cause the leaf values to have a too tight restriction as the iteration grows.

Algorithm~\ref{alg:priv_gbdts} shows our boosting framework. Given the total number of trees $T$ and the number of trees inside an ensemble $T_e$, we can get the total number of ensembles $N_e = \lceil T/T_e \rceil$. Then, the privacy budget for each tree is $(\varepsilon / N_e)$. When building the $t$-th differentially private decision tree, we first calculate the position of the tree in the ensemble as $t_e = t \mod T_e$. Since the maximum leaf value of each tree is different in GLC, to utilize the contribution of the front trees, the number of instances allocated to a tree is proportional to its leaf sensitivity. Specifically, we randomly choose $(\frac{|\mc{D}|\eta(1-\eta)^{t_e}}{1 - (1-\eta)^{T_e}})$ unused instances from the dataset $I$ as the input, where $I$ is initialized to the entire training dataset $\mc{D}$ at the beginning of an ensemble. Each tree is built using TrainSingleTree in Algorithm~\ref{alg:priv_tree}. All the $T$ trees are trained one by one, and these trees constitute our final learned model.

\begin{algorithm}
\SetNoFillComment
\DontPrintSemicolon
\KwIn{$\mc{D}$: Dataset, $\mathit{Depth}_\mathit{max}$: maximum depth}
\KwIn{$\varepsilon$: privacy budget, $\lambda$: regularization parameter}
\KwIn{$T$: total number of trees, $l$: loss function}
\KwIn{$T_e$: number of trees in an ensemble}
$N_e \leftarrow \lceil T / T_e \rceil$ \tcp*{the number of ensembles}
$\varepsilon_e \leftarrow \varepsilon / N_e$ \tcp*{privacy budget for each tree}
\For {$t = 1$ \KwTo $T$}{
    Update gradients of all training instances on loss $l$.\\
    $t_e \leftarrow t \mod T_e$ \\
    \If {$t_e == 1$}{ 
        $I \leftarrow \mc{D}$ \tcp*{initialize the dataset for an ensemble}
    }
    Randomly pick $(\frac{|\mc{D}|\eta(1-\eta)^{t_e}}{1 - (1-\eta)^{T_e}})$ instances from $I$ to constitute the subset $I_t$.\\
    $I \leftarrow I - I_t $\\
    TrainSingleTree (dataset = $I_t$,  \\
    \quad \quad \quad \quad  maximum depth = $\mathit{Depth}_\mathit{max}$, \\
    \quad \quad \quad \quad   privacy budget = $\varepsilon_e$,\\
    \quad \quad \quad \quad  $\Delta G = 3{g_l^*}^2$, \\
    \quad \quad \quad \quad  $\Delta V = \min(\frac{g_l^*}{1+\lambda}$, $2g_l^*(1-\eta)^{t-1})$).
}
\KwOut{$\varepsilon$-differentially private GBDTs}
\caption{Train differentially private GBDTs}
\label{alg:priv_gbdts}
\end{algorithm}

\begin{theorem}
The output of Algorithm~\ref{alg:priv_gbdts} satisfies $\varepsilon$-differential privacy.
\end{theorem}

\begin{proof}
Since the trees in an ensemble have disjoint inputs, the privacy budget consumption of an ensemble is still $(\varepsilon / N_e)$ due to the parallel composition. Since there are $N_e$ ensembles in total, according to sequential composition, the total privacy budget consumption is $\varepsilon / N_e * N_e = \varepsilon$.
\end{proof}

\section{Experiments}
In this section, we evaluate the effectiveness and efficiency of DPBoost. We compare DPBoost with three other approaches: 1) {\bf NP} (the vanilla GBDT): Train GBDTs without privacy concerns. 2) {\bf PARA}: A recent approach~\cite{zhao2018inprivate} that adopts parallel composition to train multiple trees, and uses only a half of unused instances when training a differentially private tree. 3) {\bf SEQ}: we extend the previous approach on decision trees~\cite{liu2018differentially} that aggregates differentially private decision trees using sequential composition. To provide sensitivities bounds for the baselines, we set $\Delta G = 3{g^*}^2$ and $\Delta V = \frac{g_m}{1+\lambda}$ in SEQ and PARA using our GDF technique.

We implemented DPBoost based on LightGBM\footnote{\url{https://github.com/microsoft/LightGBM}}. Our experiments are conducted on a machine with one Xeon W-2155 10 core CPU. We use 10 public datasets in our evaluation. The details of the datasets are summarized in Table~\ref{tbl:datasets}. There are eight real-world datasets and two synthetic datasets (i.e., synthetic\_cls and synthetic\_reg). The real-world datasets are available from the LIBSVM website\footnote{\url{https://www.csie.ntu.edu.tw/~cjlin/libsvmtools/datasets/}}. The synthetic datasets are generated using scikit-learn\footnote{\url{https://scikit-learn.org/stable/datasets/index.html#sample-generators}}~\cite{pedregosa2011scikit}. We show test errors and RMSE (root mean square error) for the classification and regression task, respectively. The learning rate is set to 0.01. The maximum depth is set to 6. The regularization parameter $\lambda$ is set to 0.1. We use square loss function and the threshold $g_l^*$ is set to 1. For the regression task, the labels are scaled into $[-1,1]$ before training (the reported RMSE is re-scaled). For the classification task, we set the second-order gradient $h$ to the constant value one in the training to apply our theoretical analysis. We use 5-fold cross-validation for model evaluation. The number of trees inside an ensemble is set to 50 in DPBoost. We have also tried the other settings for the number of trees inside an ensemble (e.g., 20 and 40). The experiments are available in Appendix C.

\begin{table}
\centering
\caption{Datasets used in the experiments.}
\label{tbl:datasets}
\resizebox{0.95\columnwidth}{!}{%
\begin{tabular}{|c|c|c|c|}
\hline
datasets & \#data & \#features & task \\ \hline
adult & 32,561 & 123 & \multirow{7}{*}{classification} \\ \cline{1-3}
real-sim & 72,309 & 20,958 &  \\ \cline{1-3}
covtype & 581,012 & 54 & \\ \cline{1-3}
susy & 5,000,000 & 18 &  \\ \cline{1-3}
cod-rna & 59,535 & 8 & \\ \cline{1-3}
webdata & 49,749 & 300 &  \\ \cline{1-3}
synthetic\_cls & 1,000,000 & 400 & \\ \hline
abalone & 4,177 & 8 & \multirow{3}{*}{regression} \\ \cline{1-3}
YearPredictionMSD & 463,715 & 90 &  \\ \cline{1-3}
synthetic\_reg & 1,000,000 & 400 & \\ \hline
\end{tabular}%
}
\end{table}

\subsection{Test Errors}
\label{sec:test_error}
We first set the number of ensembles to one in DPBoost and the number of trees to 50 for all approaches. Figure~\ref{fig:error_budget} shows the test errors of four approaches with different $\varepsilon$ (i.e., 1, 2, 4, 6, 8, and 10). We have the following observations. First, SEQ has the worst performance on almost all cases. With only sequential composition, each tree in SEQ gets a very small privacy budget. The noises in SEQ are huge and lead to high test errors in the prediction. Second, DPBoost can always outperform PARA and SEQ especially when the given budget is small. DPBoost outperforms PARA, mainly because our tightening bounds on sensitivity allows us to use a smaller noise to achieve differential privacy. When the privacy budget is one, DPBoost can achieve 10\% lower test errors on average in the classification task and significant reduction on RMSE in the regression tasks. Moreover, the variance of DPBoost is usually close to zero. DPBoost is very stable compared with SEQ and PARA. Last, DPBoost can achieve competitive performance compared with NP. The results of DPBoost and NP are quite close in many cases, which show high model quality of our design.

\begin{figure*}[t]
\centering
\subfloat[adult]{\includegraphics[width=0.19\textwidth]{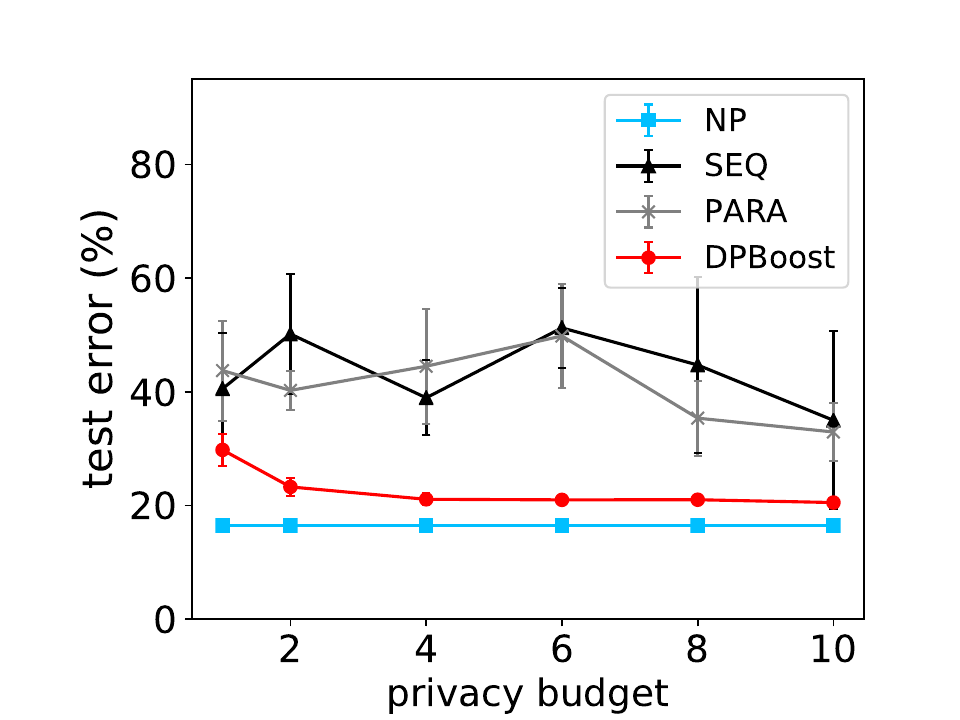}%
}
\subfloat[real-sim]{\includegraphics[width=0.19\textwidth]{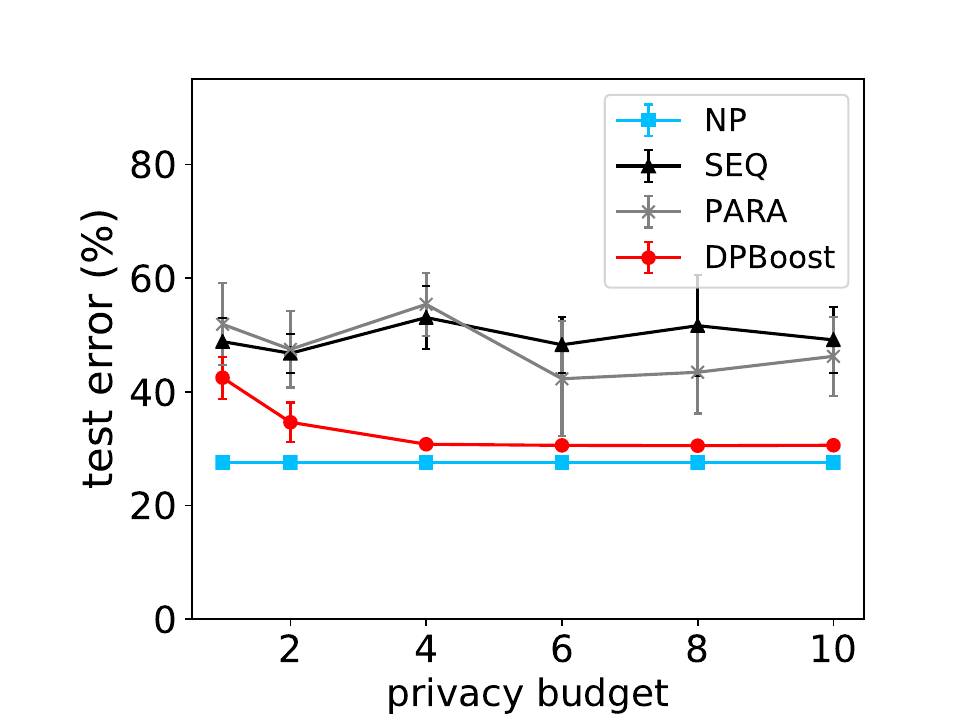}%
}
\subfloat[covtype]{\includegraphics[width=0.19\textwidth]{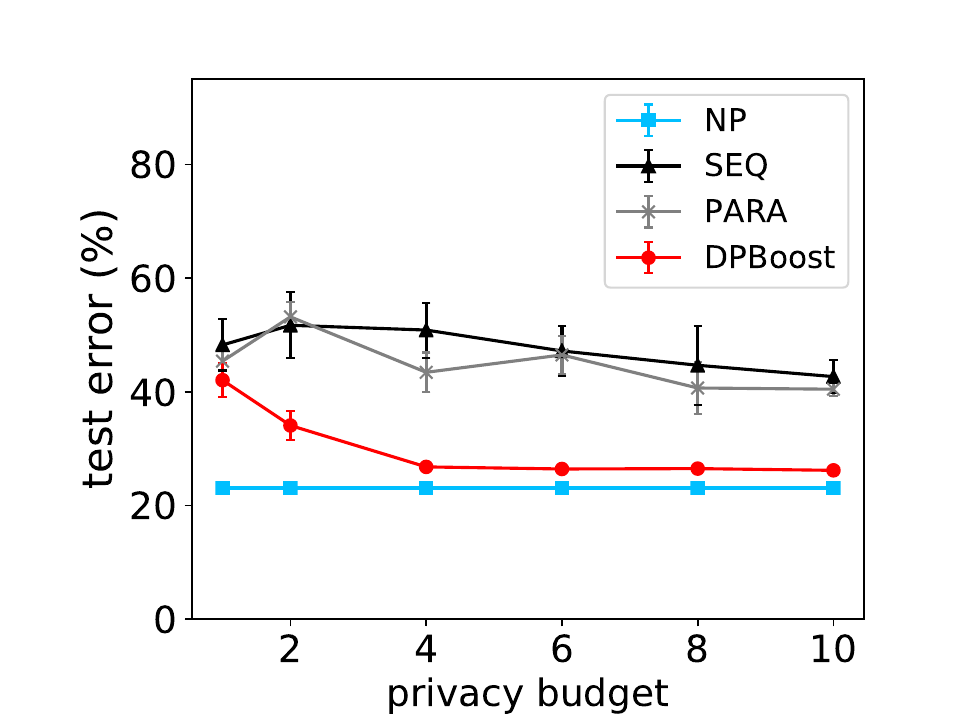}%
}
\subfloat[susy]{\includegraphics[width=0.19\textwidth]{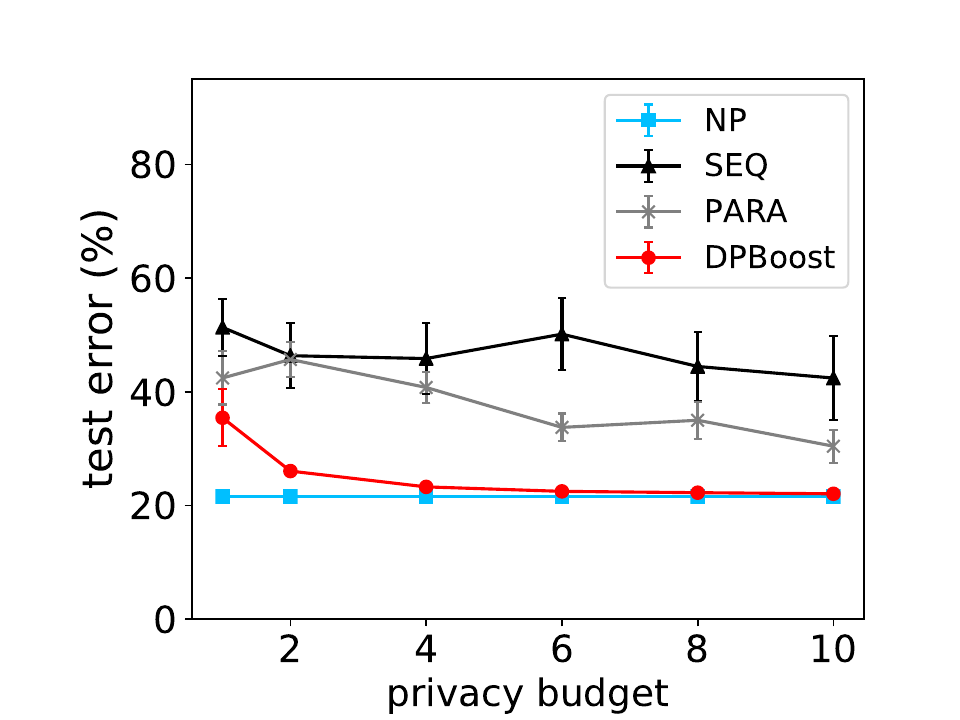}%
}
\subfloat[cod-rna]{\includegraphics[width=0.19\textwidth]{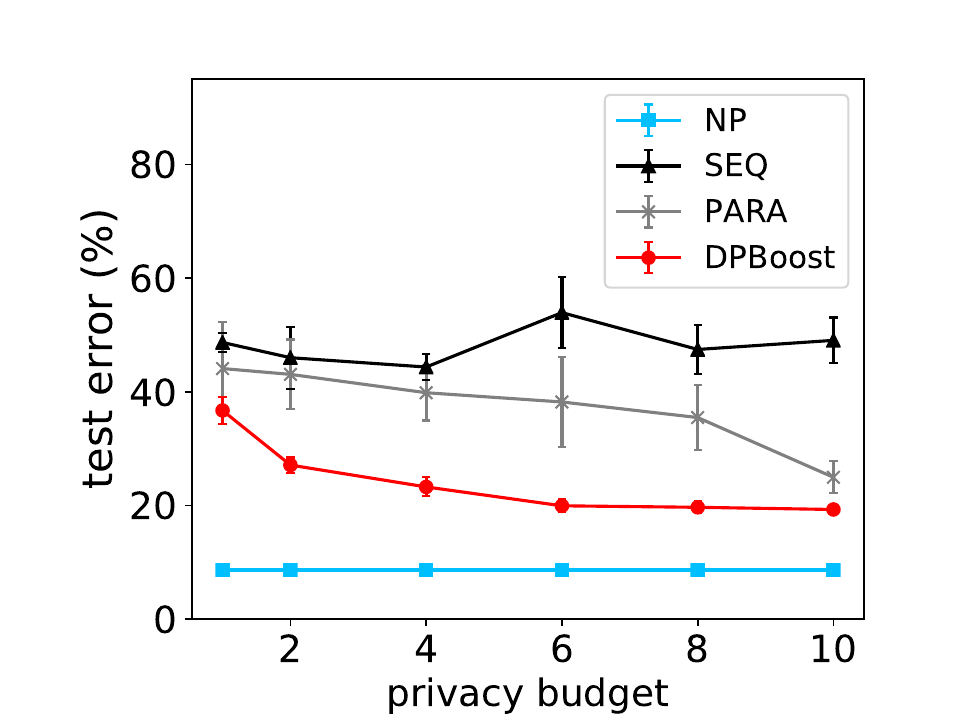}
}
\hfil
\subfloat[webdata]{\includegraphics[width=0.19\textwidth]{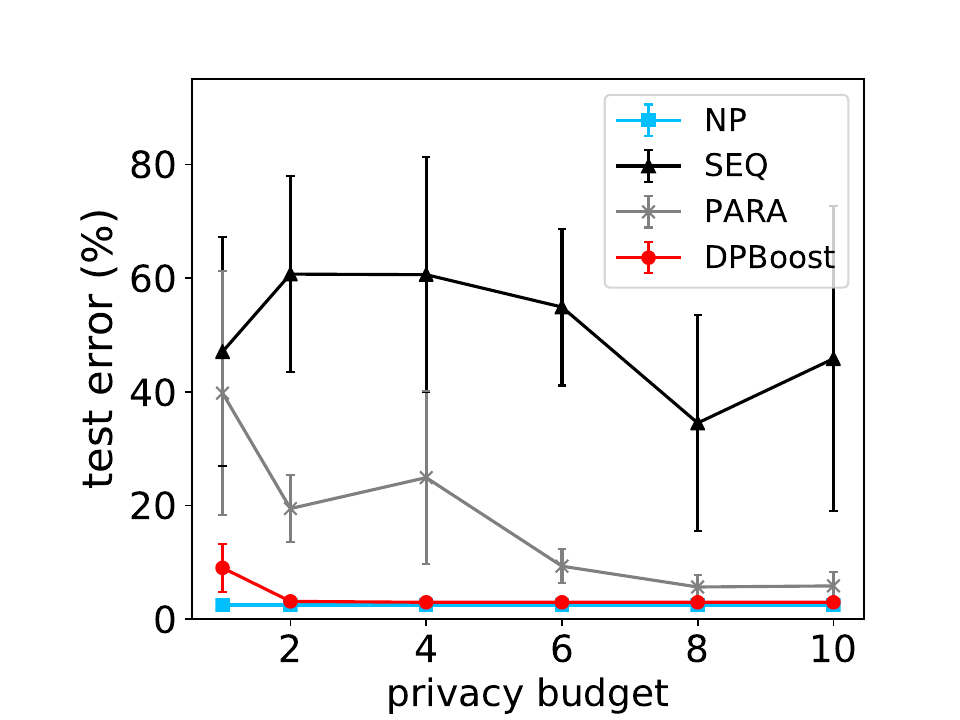}%
}
\subfloat[synthetic\_cls]{\includegraphics[width=0.19\textwidth]{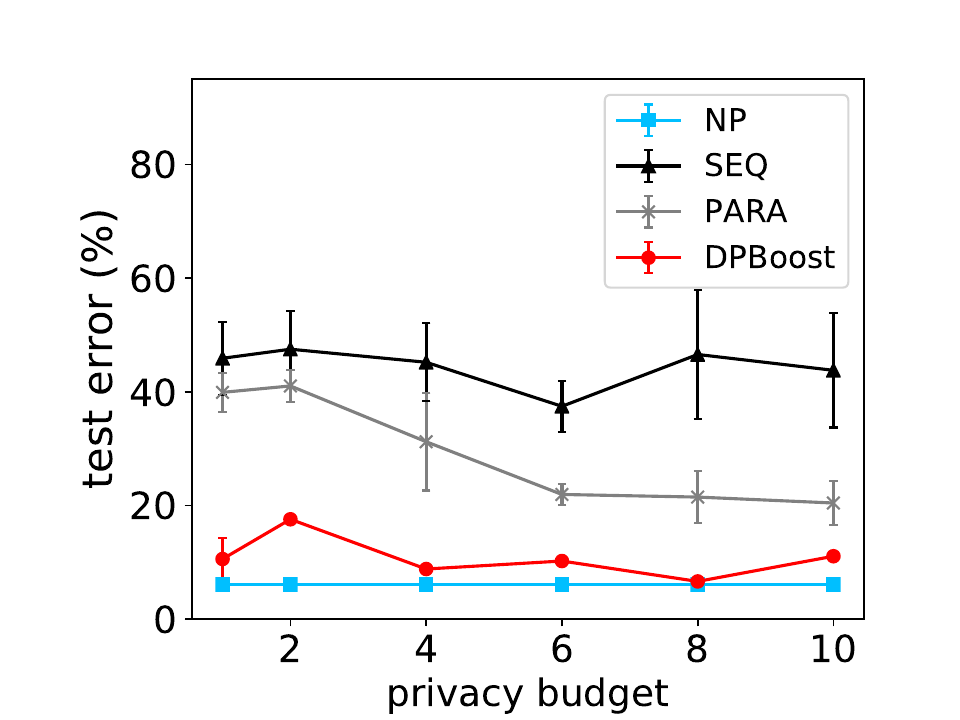}%
}
\subfloat[abalone]{\includegraphics[width=0.19\textwidth]{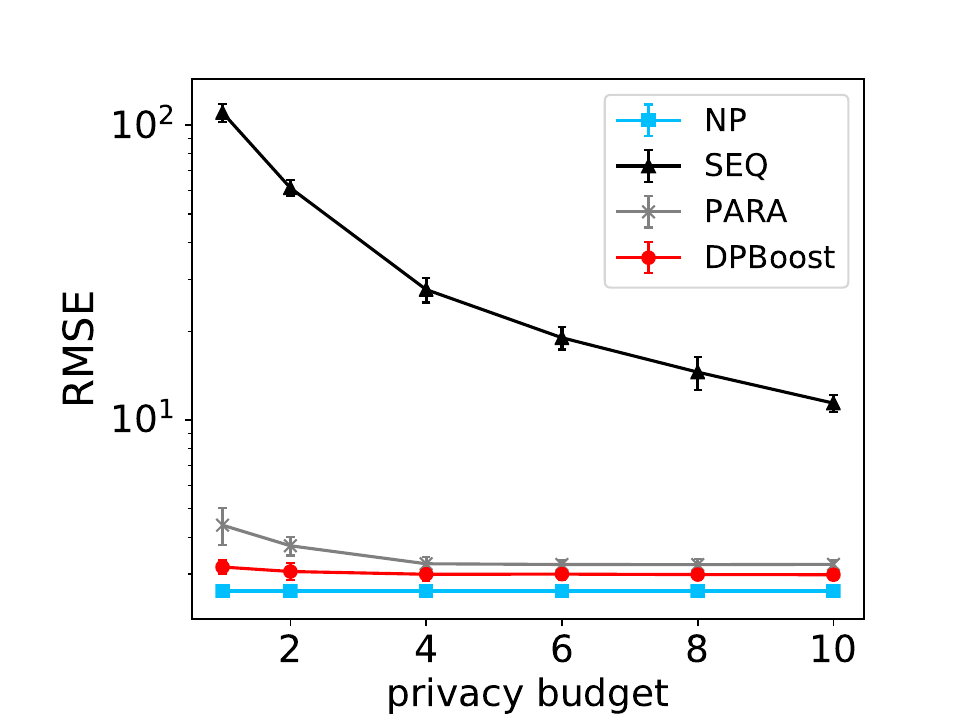}%
}
\subfloat[YearPredictionMSD]{\includegraphics[width=0.19\textwidth]{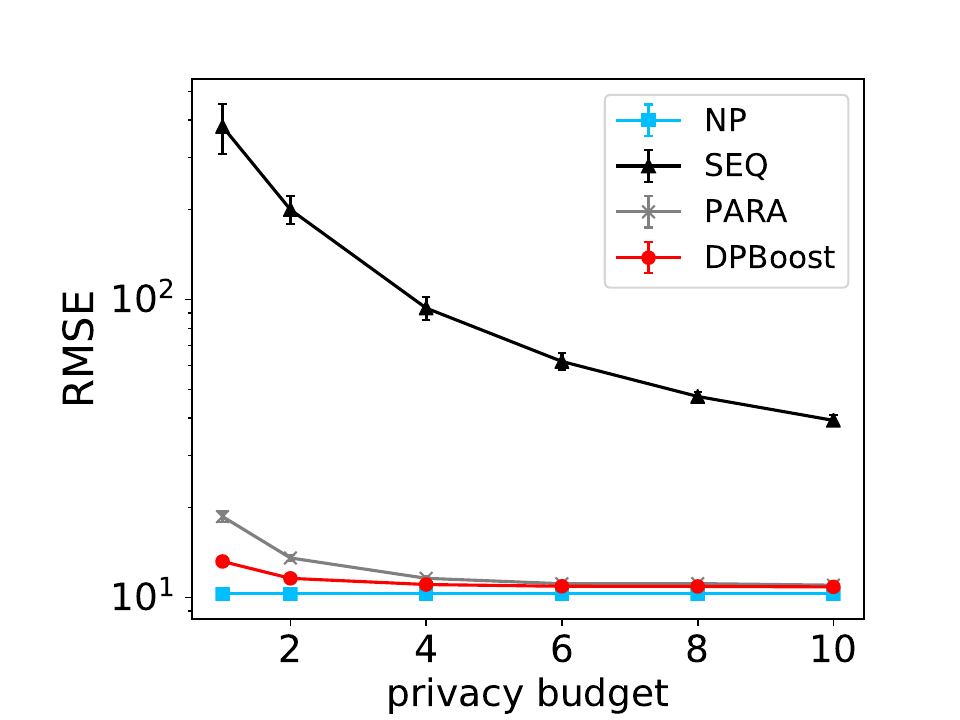}%
}
\subfloat[synthetic\_reg]{\includegraphics[width=0.19\textwidth]{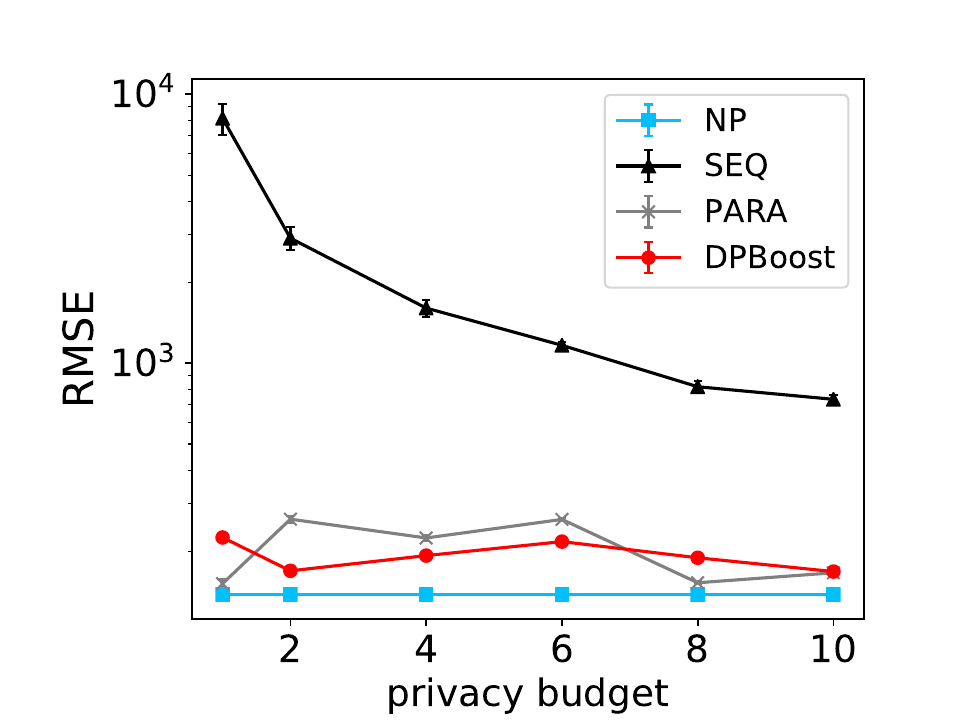}%
}
\caption{Comparison of the test errors/RMSE given different total privacy budgets. The number of trees is set to 50.}
\label{fig:error_budget}
\end{figure*}

\begin{figure*}[t]
\centering
\subfloat[adult]{\includegraphics[width=0.19\textwidth]{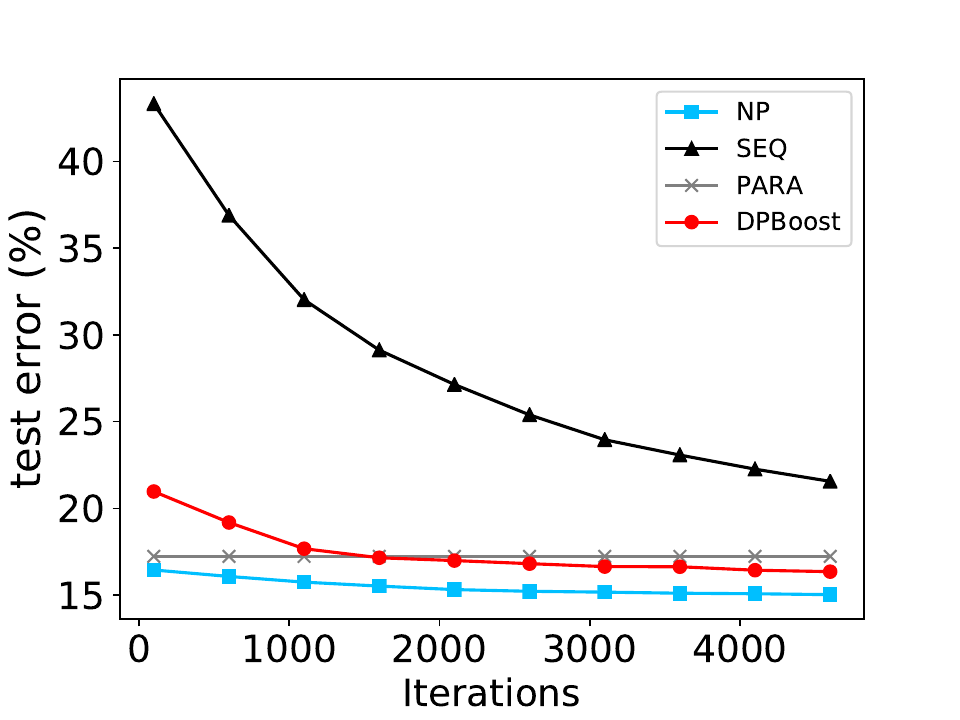}%
}
\subfloat[real-sim]{\includegraphics[width=0.19\textwidth]{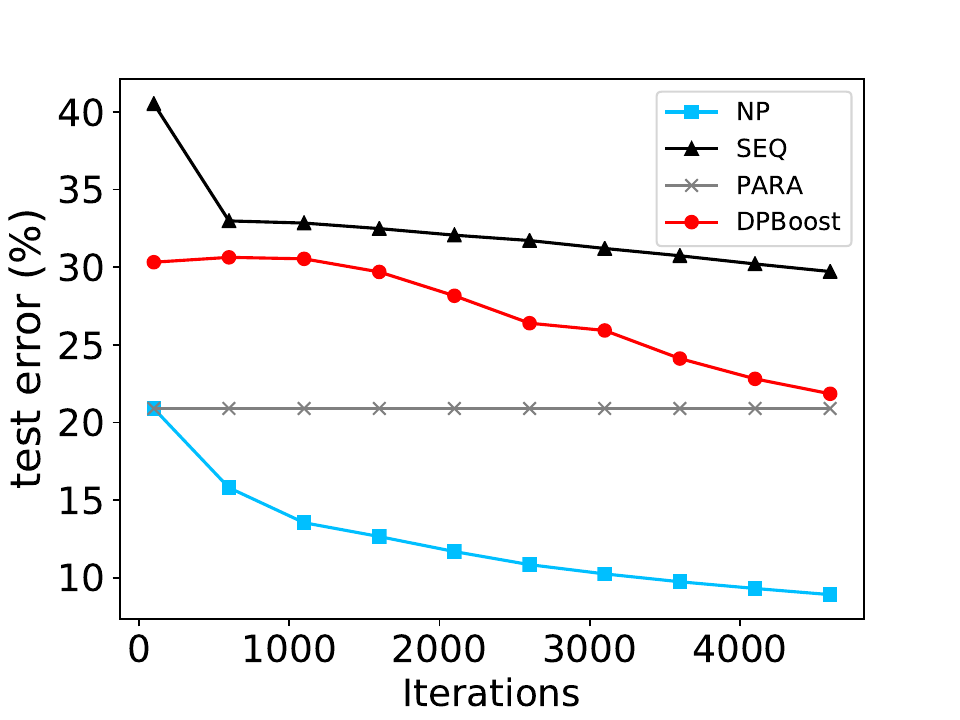}%
}
\subfloat[covtype]{\includegraphics[width=0.19\textwidth]{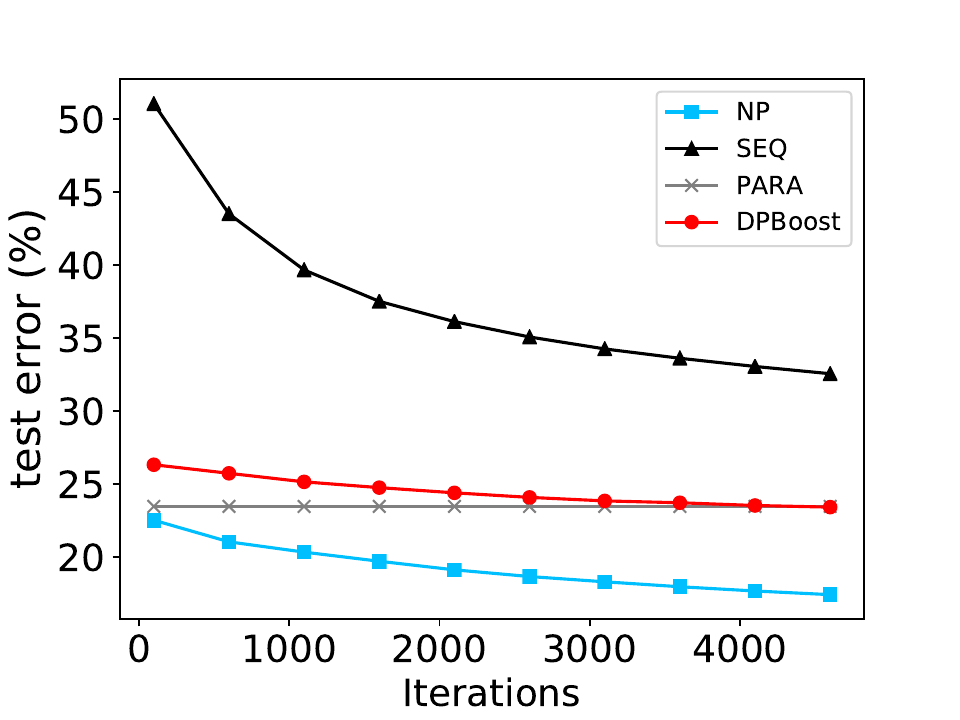}%
}
\subfloat[susy]{\includegraphics[width=0.19\textwidth]{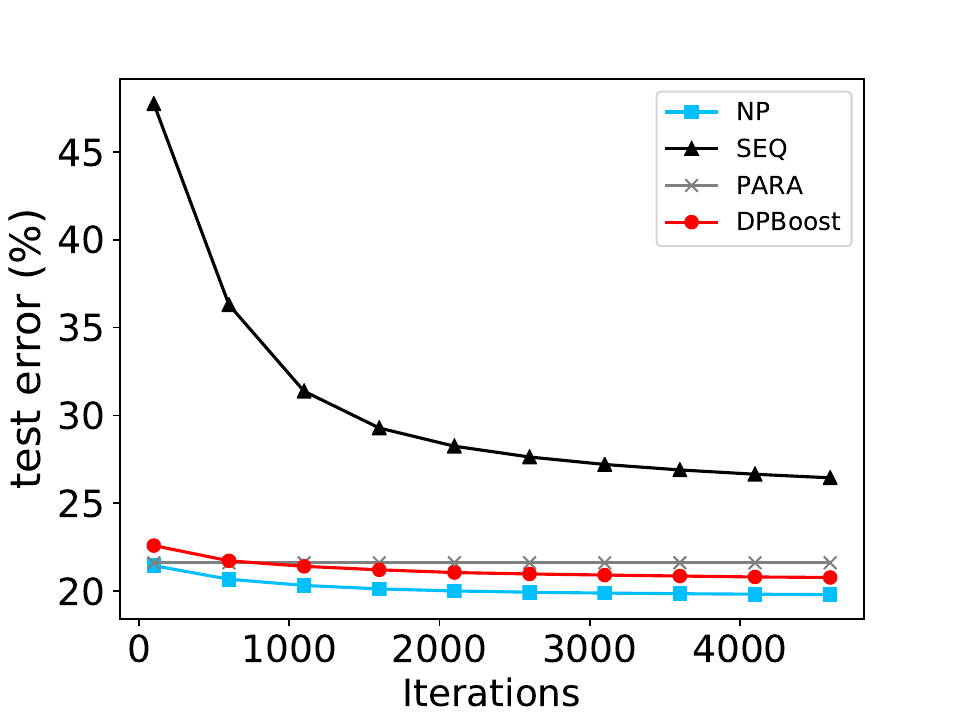}%
}
\subfloat[cod-rna]{\includegraphics[width=0.19\textwidth]{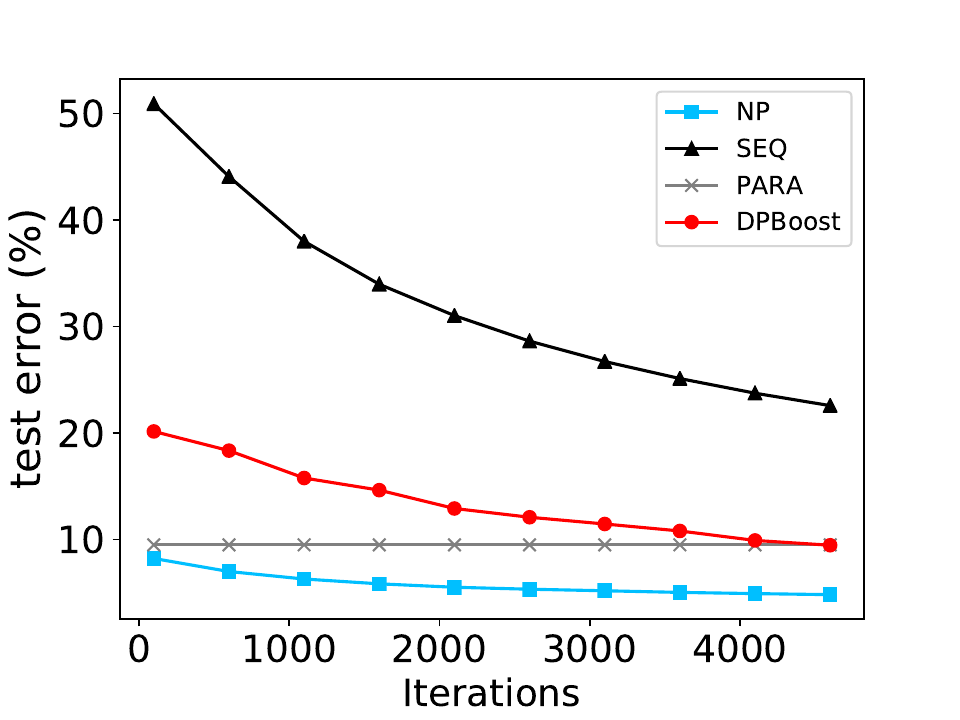}%
}
\caption{Comparison of test error convergence. The number of trees is set to 1000.}
\label{fig:convergence}
\end{figure*}

To show the effect of boosting, we increase the number of ensembles to 100 and the maximum number of trees to 5000. The privacy budget for each ensemble is set to 5. For fairness, the total privacy budget for SEQ and PARA is set to 500 to achieve the same privacy level as DPBoost. We choose the first five datasets as representatives. Figure~\ref{fig:convergence} shows the convergence curves of four approaches. First, since the privacy budget for each tree is small, the errors of SEQ are very high. Second, since PARA takes a half of the unused instances at each iteration, it can only train a limited number of trees until the unused instances are too few to train an effective tree (e.g., about 20 trees for dataset \emph{SUSY}). Then, the curve of PARA quickly becomes almost flat and the performance cannot increase as the iteration grows even given a large total privacy budget. Last, DPBoost has quite good behavior of reducing test errors as the number of trees increases. DPBoost can continue to decrease the accuracy loss and outperform PARA and SEQ on four datasets with 5000 trees, which demonstrate the effectiveness of our privacy budget allocation. Also, DPBoost can preserve the effect of boosting well.

\begin{table}[h]
\centering
\caption{Training time per tree (second) of DPBoost and NP.}
\label{tbl:time}
\begin{tabular}{|c|c|c|}
\hline
datasets & DPBoost & NP  \\ \hline
adult & 0.019 & 0.007  \\ \hline
real-sim & 2.97 & 0.82    \\ \hline
covtype & 0.085 & 0.044 \\ \hline
SUSY & 0.38 & 0.32  \\ \hline
cod-rna & 0.016 & 0.009 \\ \hline
webdata & 0.032 & 0.013  \\ \hline
synthetic\_cls & 1.00 & 0.36 \\ \hline
abalone & 2.95 & 2.85  \\ \hline
YearPrediction & 0.38 & 0.12  \\ \hline
synthetic\_reg & 0.96 & 0.36  \\ \hline
\end{tabular}%
\end{table}
\subsection{Training Time Efficiency}
We show the training time comparison between DPBoost and NP. The computation overhead of our approach mainly comes from the exponential mechanism, which computes a probability for each gain. Thus, this overhead depends on the number of split values and increases as the number of dimensions of training data increases. Table~\ref{tbl:time} shows the average training time per tree of DPBoost and NP. The setting is the same as the second experiment of Section~\ref{sec:test_error}. The training time per tree of DPBoost is comparable to NP in many cases (meaning that the overhead can be very small), or about 2 to 3 times slower than NP in other cases. Nevertheless, the training of DPBoost is very fast. The time per tree of DPBoost is no more than 3 seconds in those 10 datasets.

\section{Conclusions}
Differential privacy has been an effective mechanism for protecting data privacy. Since GBDT has become a popular and widely used training system for many machine learning and data mining applications, we propose a differentially private GBDT algorithm called DPBoost. It addresses the limitations of previous works on serious accuracy loss due to loose sensitivity bounds and ineffective privacy budget allocations. Specifically, we propose gradient-based data filtering and geometric leaf clipping to control the training process in order to tighten the sensitivity bound. Moreover, we design a two-level boosting framework to well exploit both the privacy budget and the effect of boosting. Our experiments show the effectiveness and efficiency of DPBoost. 

\section*{Acknowledgements}
This work is supported by a MoE AcRF Tier 1 grant (T1 251RES1824) and a MOE Tier 2 grant (MOE2017-T2-1-122) in Singapore. The authors thank Xuejun Zhao for her discussion and Shun Zhang for his comments to improve the paper.
\bibliographystyle{aaai}
\bibliography{aaai20}

\appendix

\section{Proof of Lemma 2}
\begin{proof}
Consider two adjacent instance sets $I_1 = \{\mb{x}_i\}_{i=1}^{n}$ and $I_2 = I_1 \cup \{\mb{x}_s\}$ that differ in a single instance. We have
\begin{equation}
\begin{aligned}
\Delta V &= |-\frac{\sum_{i=1}^n g_i}{n+\lambda}-(-\frac{\sum_{i=1}^{n} g_i + g_{s}}{n+1+\lambda})| \\
    &= |\frac{(n+\lambda) g_{s}-\sum_{i=1}^n g_i}{(n+\lambda)(n+1+\lambda)} | \\
\end{aligned}
\end{equation}
When $g_s = g^*$ and $\sum_{i=1}^n g_i = -ng^*$, the above function can achieve maximum. Thus, we have
\begin{equation}
\Delta V \leq{\frac{(2n+\lambda)g^*}{(n+\lambda)(n+1+\lambda)}} \leq \frac{g^*}{1+\lambda}
\end{equation}
\end{proof}

\section{Experimental Study on GDF and GLC}
We use the regression tasks as the examples to study the effect of gradient-based data filtering and geometric leaf clipping. Table~\ref{tbl:gdf_count} shows the number of instances using gradient-based data filtering. As we can see, the percentage of the filtered instances is small, which is no more than 8\%. Thus, the filtering strategy will not produce much approximation error according to Theorem~\ref{theo:error_GBF}.

\begin{table}
\centering
\caption{The number of training instances with/without gradient-based data filtering}
\label{tbl:gdf_count}
\resizebox{\columnwidth}{!}{%
\begin{tabular}{|c|c|c|c|}
\hline
 & w/ GDF & w/o GDF & filtered ratio \\ \hline
abalone & 3340 & 3292 & 1.44\% \\ \hline
YearPredictionMSD & 370902 & 340989 & 8.06\% \\ \hline
sklearn\_reg & 800000 & 799999 & 0 \\ \hline
\end{tabular}%
}
\end{table}

Figure~\ref{fig:geo} shows the results with and without geometric leaf clipping in our DPBoost. The geometric leaf clipping can always improve the performance of DPBoost. The improvement is quite significant in \emph{YearPrediction} and \emph{synthetic\_reg}.

\begin{figure}
\centering
\subfloat[abalone]{\includegraphics[width=0.16\textwidth]{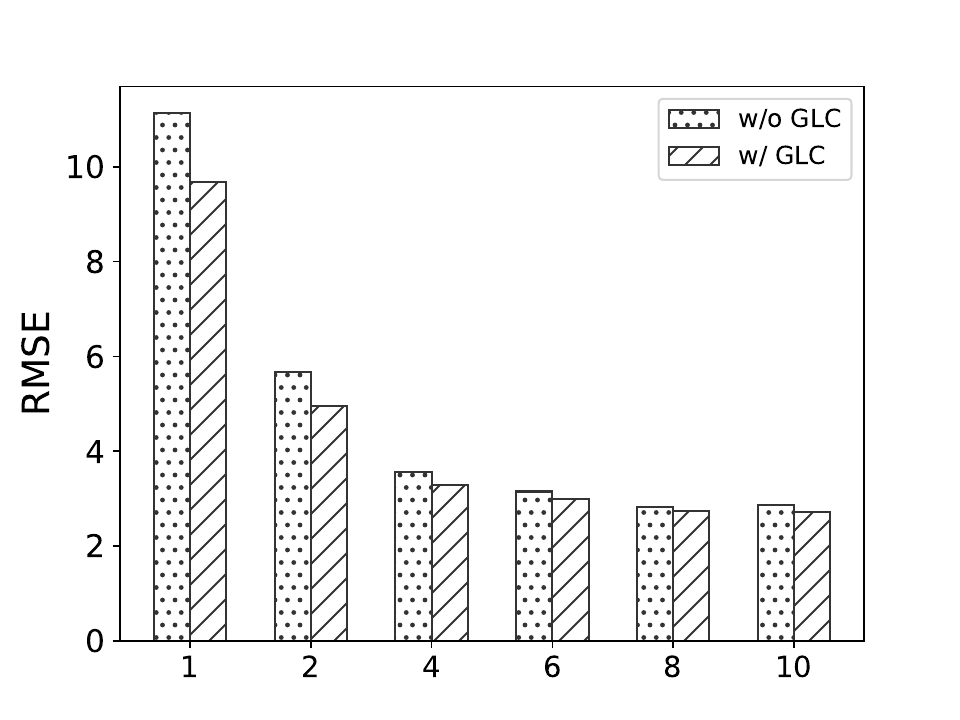}
}
\subfloat[YearPrediction]{\includegraphics[width=0.16\textwidth]{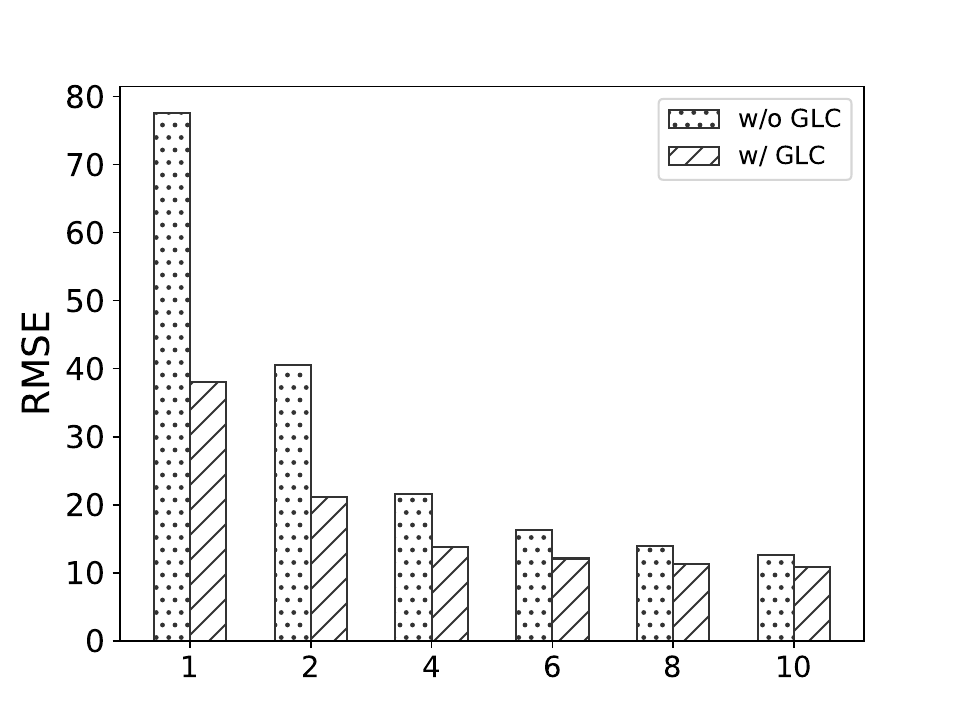}%
}
\subfloat[synthetic\_reg]{\includegraphics[width=0.16\textwidth]{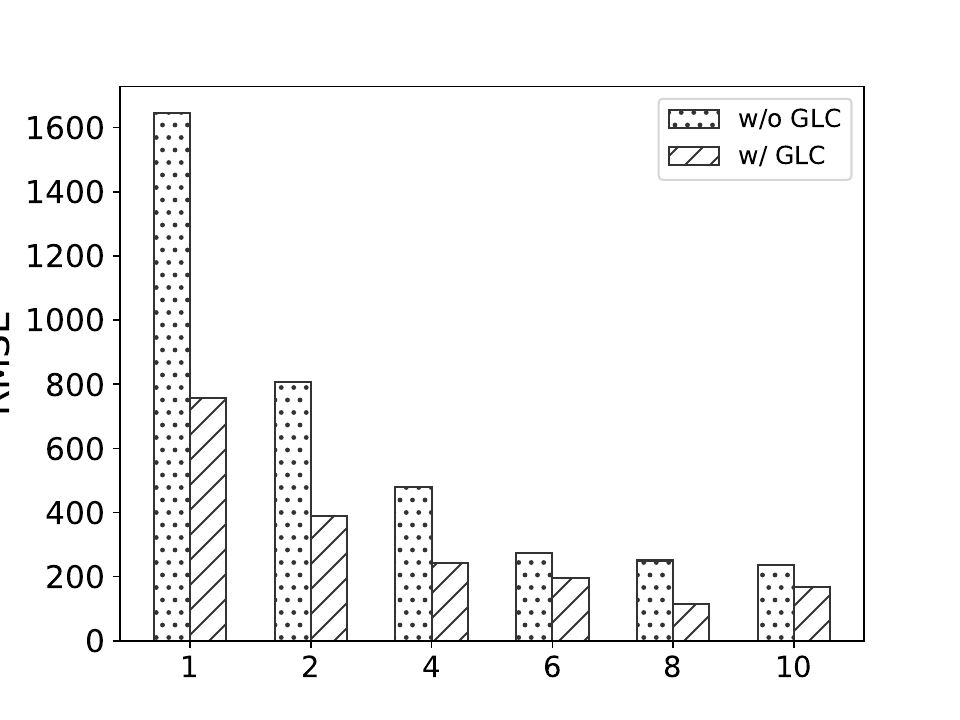}%
}
\caption{The effect of geometric leaf clipping}
\vspace{-10pt}
\label{fig:geo}
\end{figure}

\section{Additional Experiments}
Here we show the results of a different number of trees inside an ensemble (i.e., 20 and 40). The results of one ensemble are shown in Figure~\ref{fig:error_budget_20} and Figure~\ref{fig:error_budget_40}. As we can see, DPBoost can always outperform SEQ and PARA especially when the given budget is small. Furthermore, the accuracy of DPBoost is close to NP. Then we set the maximum number of trees to 1000. The results are shown in Figure~\ref{fig:convergence_20} and Figure~\ref{fig:convergence_40}. Still, DPBoost can well exploit the effect of boosting.

\begin{figure}[]
\centering
\subfloat[adult]{\includegraphics[width=0.22\textwidth]{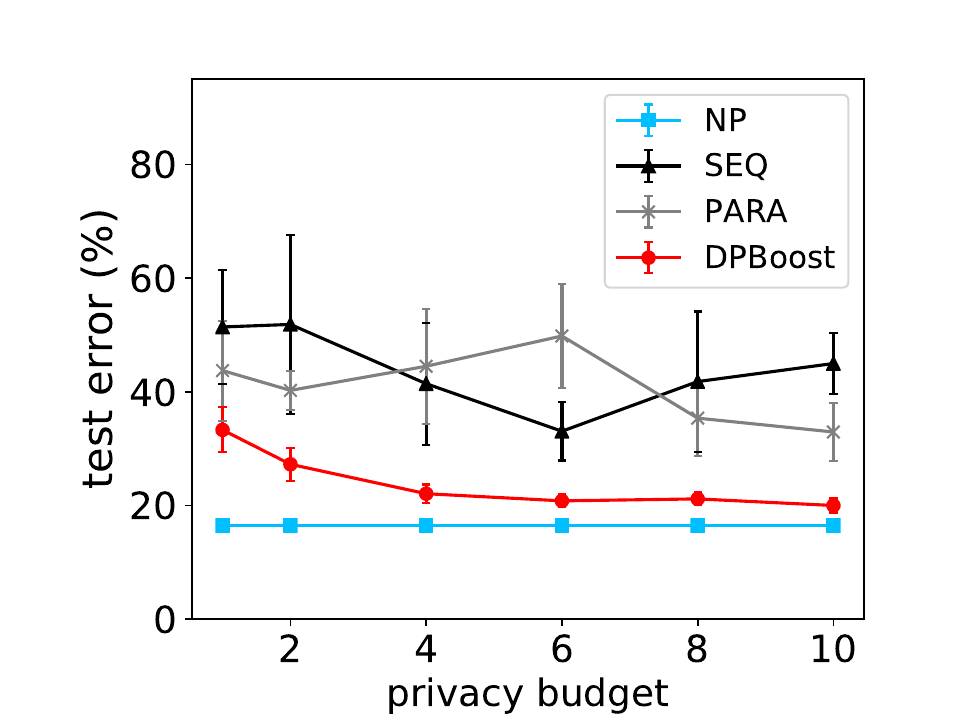}
}
\subfloat[real-sim]{\includegraphics[width=0.22\textwidth]{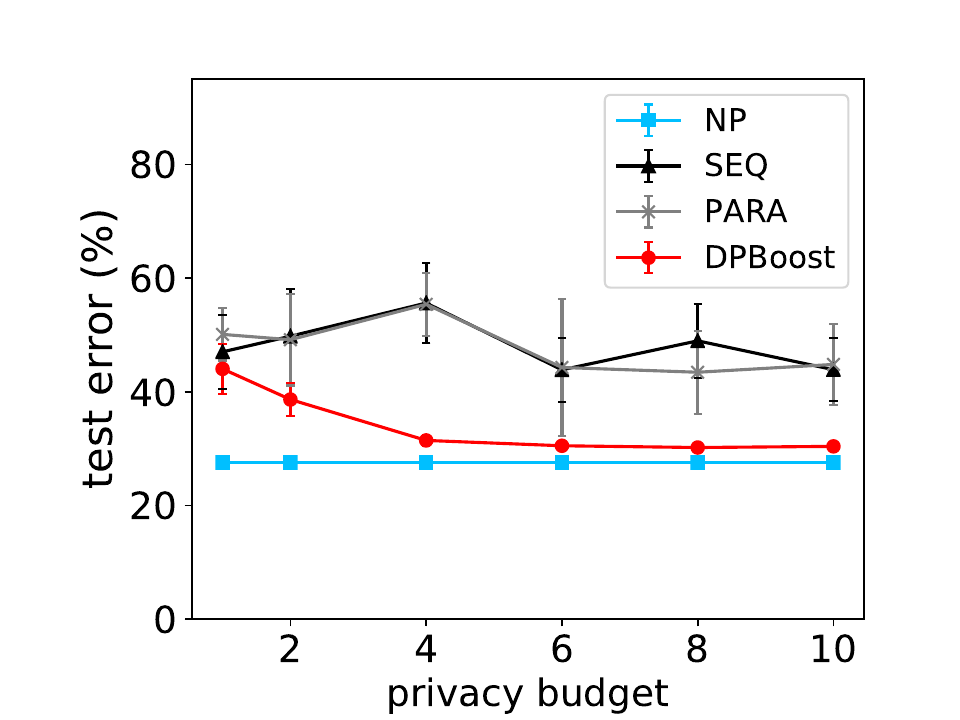}%
}
\hfil
\subfloat[covtype]{\includegraphics[width=0.22\textwidth]{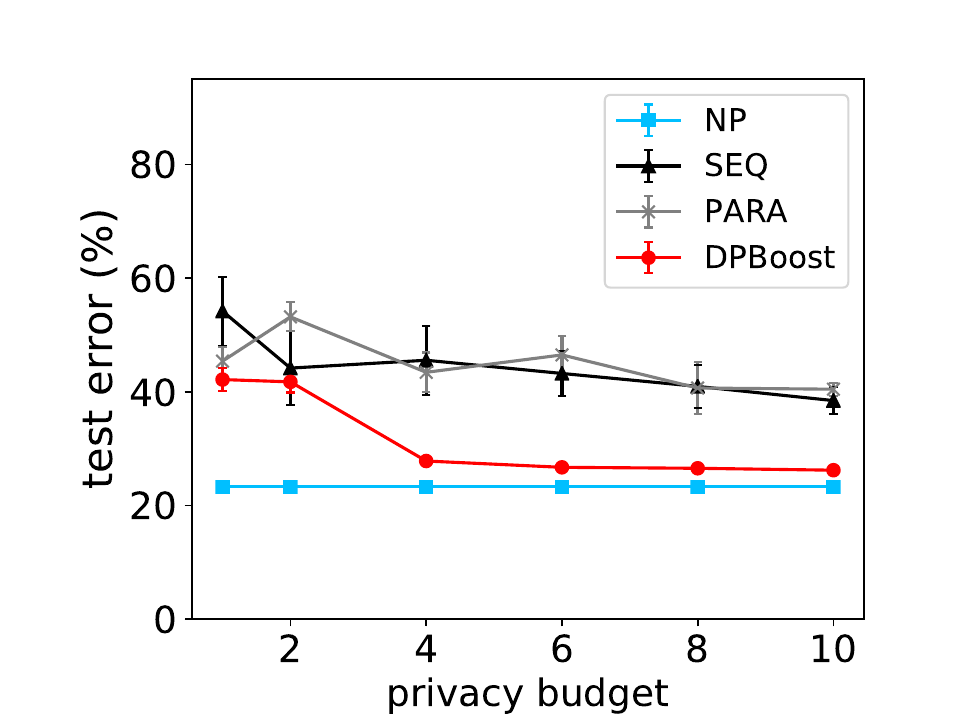}%
}
\subfloat[susy]{\includegraphics[width=0.22\textwidth]{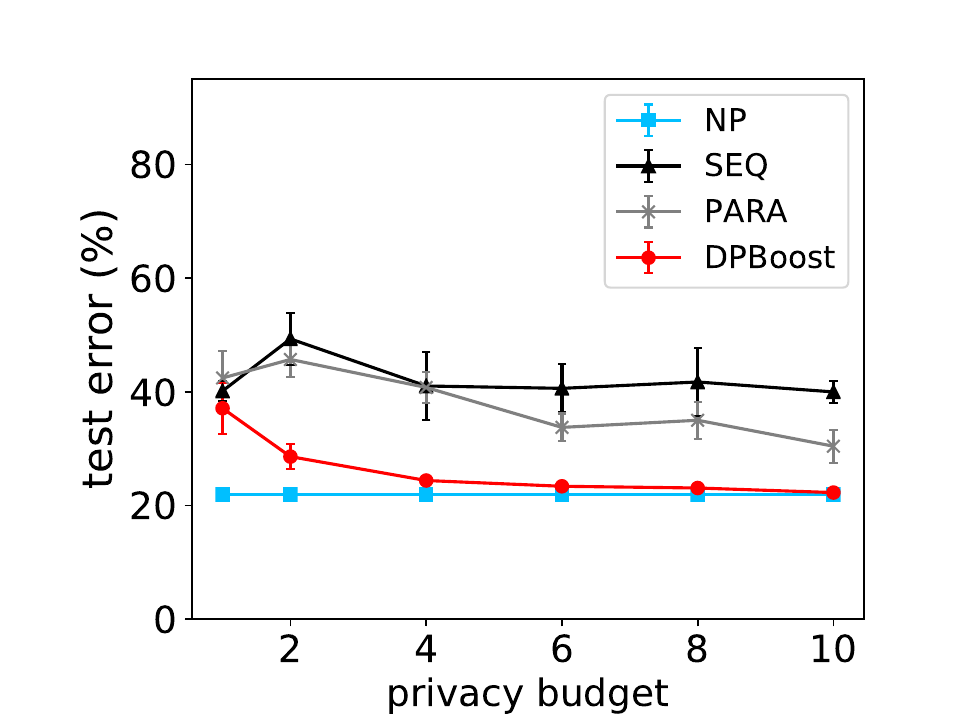}%
}
\hfil
\subfloat[cod-rna]{\includegraphics[width=0.22\textwidth]{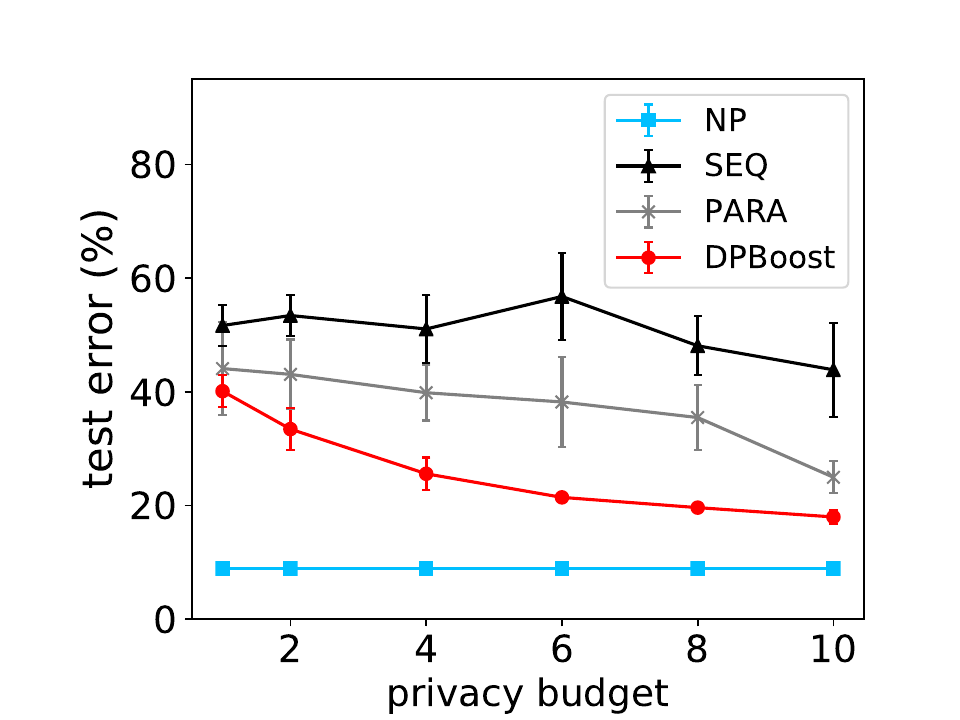}
}
\subfloat[webdata]{\includegraphics[width=0.22\textwidth]{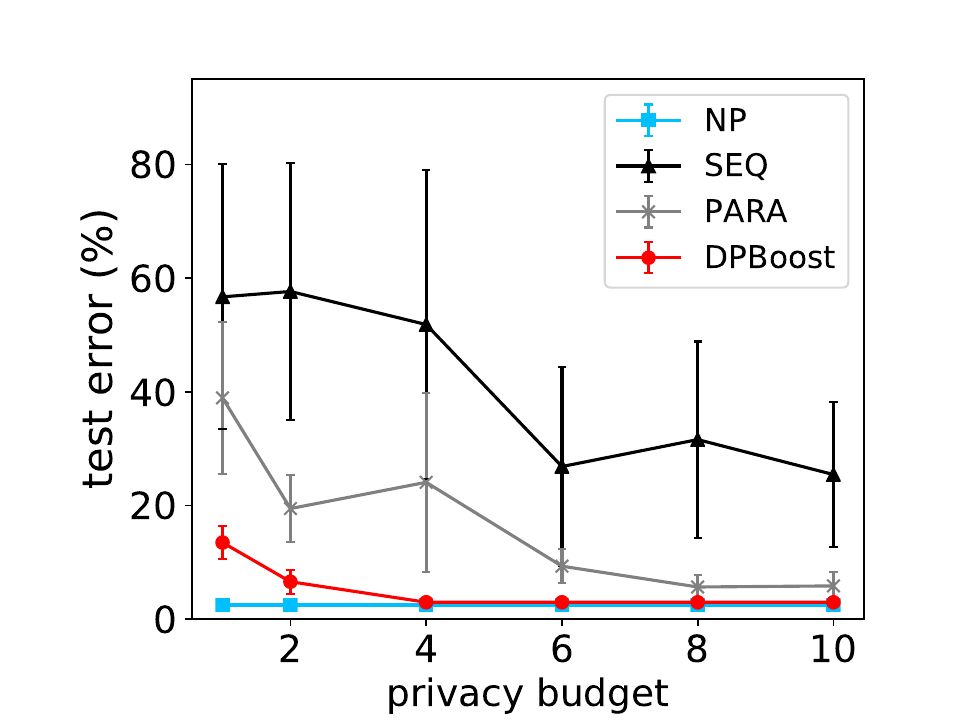}%
}
\hfil
\subfloat[synthetic\_cls]{\includegraphics[width=0.22\textwidth]{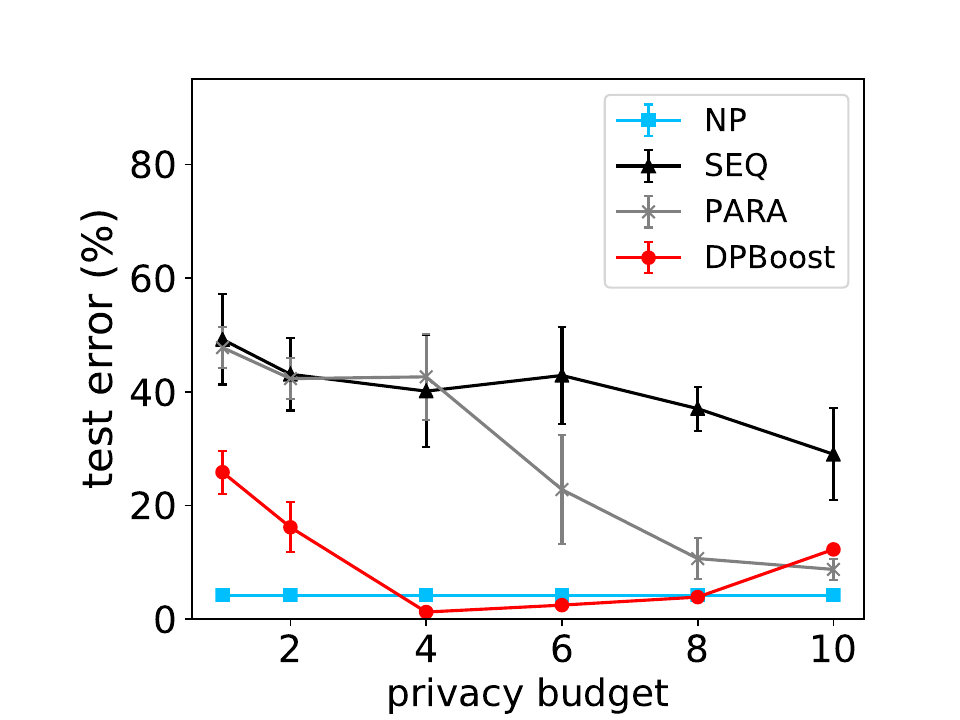}%
}
\subfloat[abalone]{\includegraphics[width=0.22\textwidth]{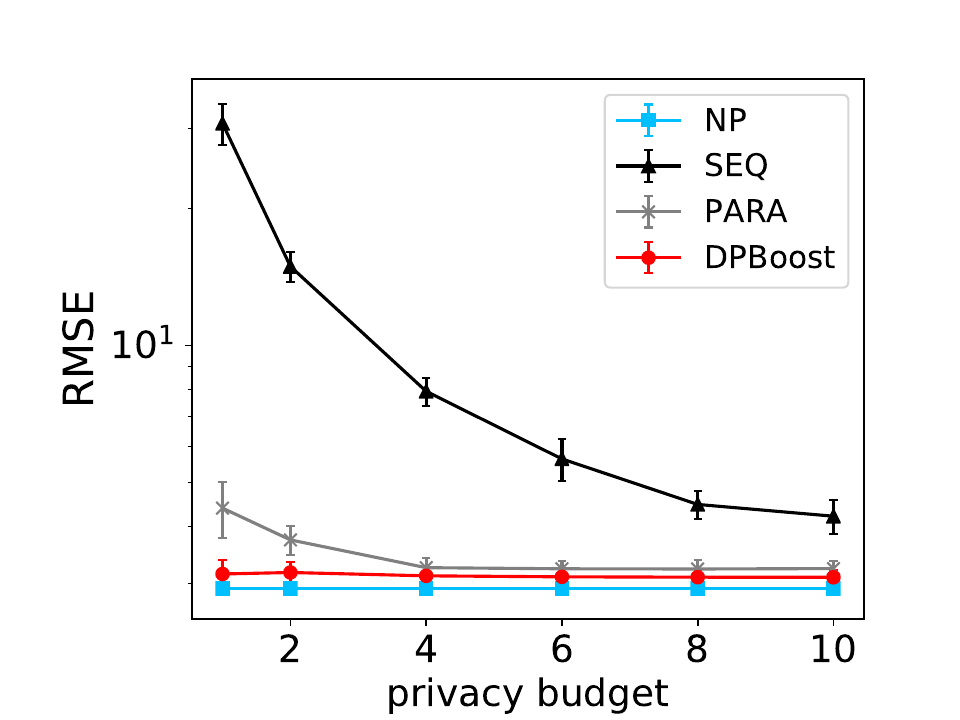}%
}
\hfil
\subfloat[YearPredictionMSD]{\includegraphics[width=0.22\textwidth]{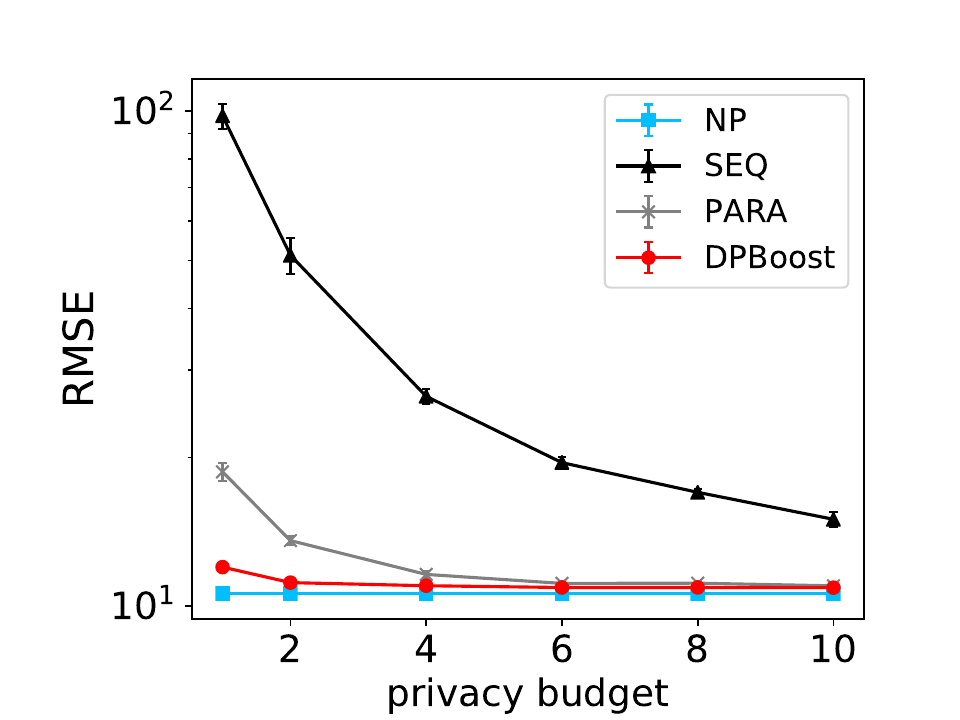}%
}
\subfloat[synthetic\_reg]{\includegraphics[width=0.22\textwidth]{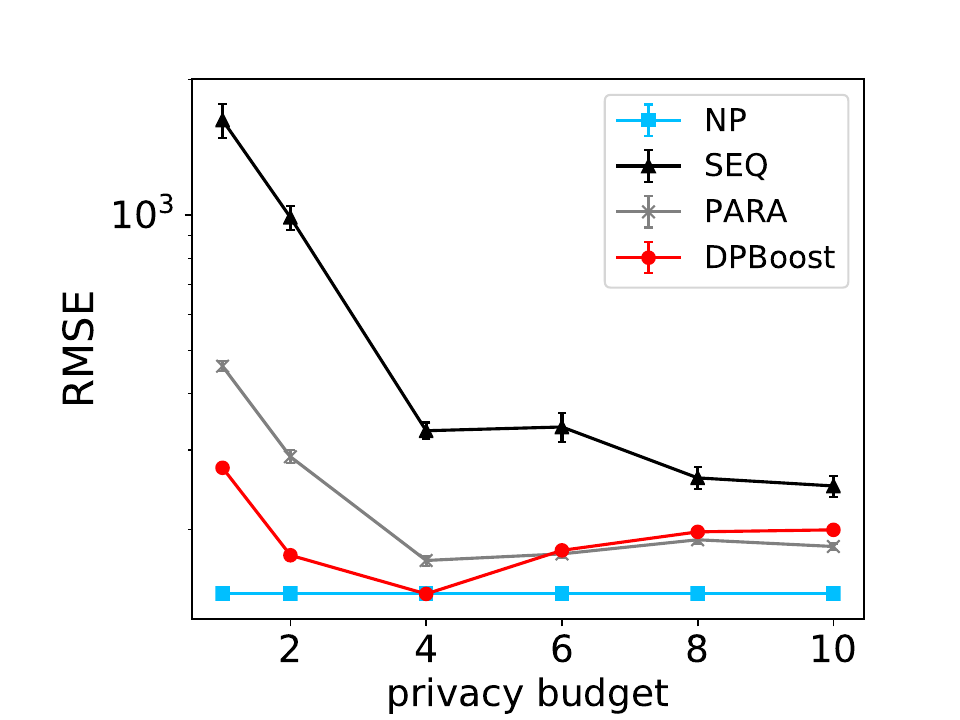}%
}
\caption{Comparison of the test errors/RMSE given different total privacy budgets. The number of trees and the number of trees inside an ensemble is set to 20.}
\vspace{-10pt}
\label{fig:error_budget_20}
\end{figure}

\begin{figure*}
\centering
\subfloat[adult]{\includegraphics[width=0.2\textwidth]{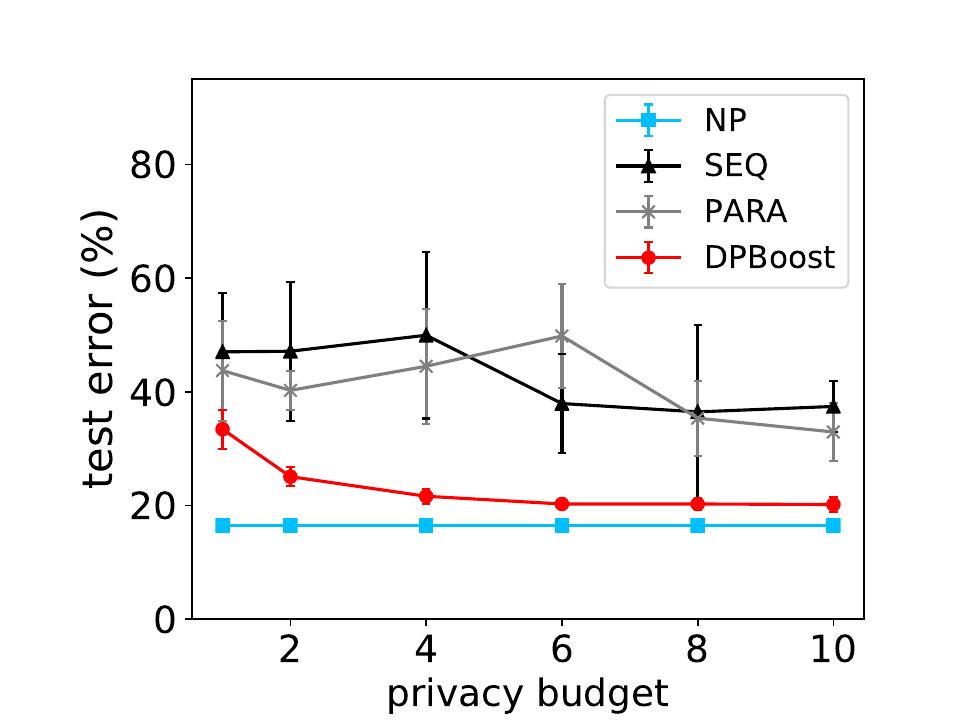}%
}
\subfloat[real-sim]{\includegraphics[width=0.2\textwidth]{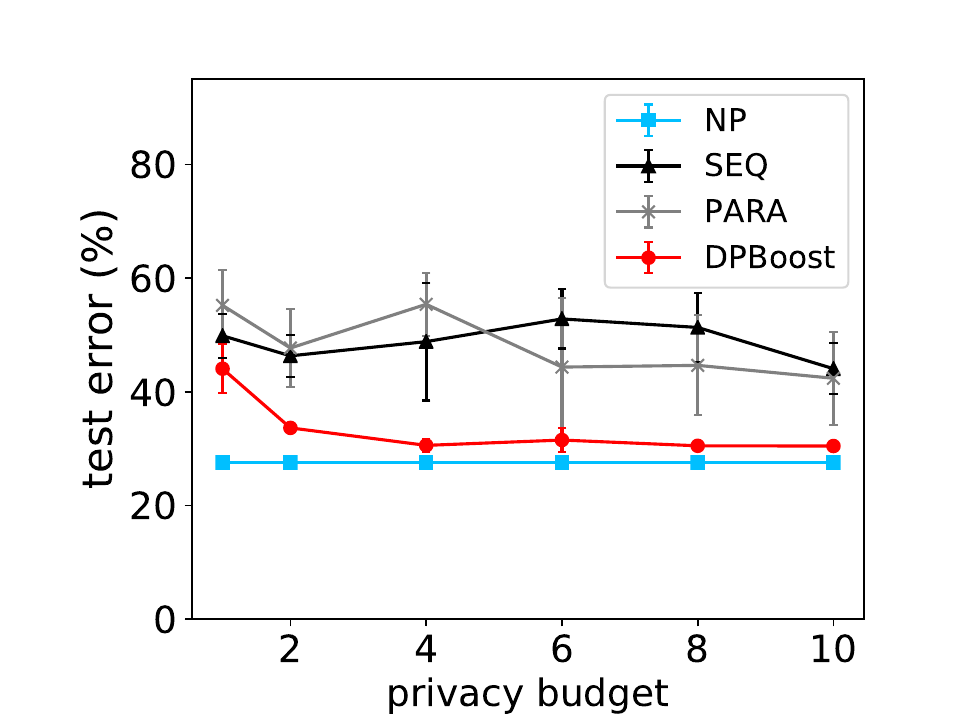}%
}
\subfloat[covtype]{\includegraphics[width=0.2\textwidth]{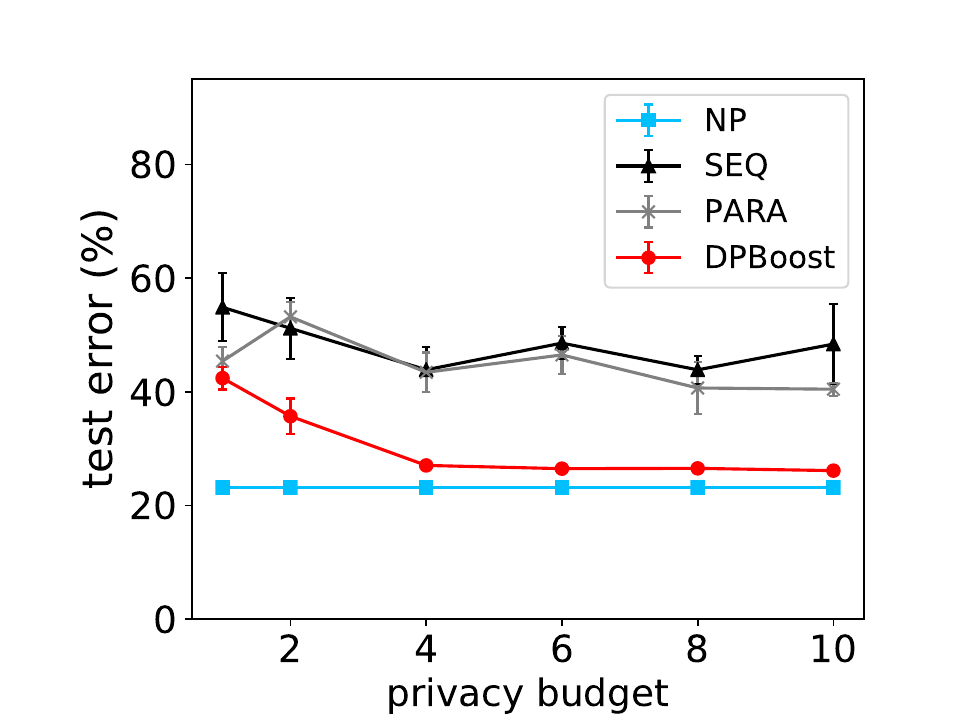}%
}
\subfloat[susy]{\includegraphics[width=0.2\textwidth]{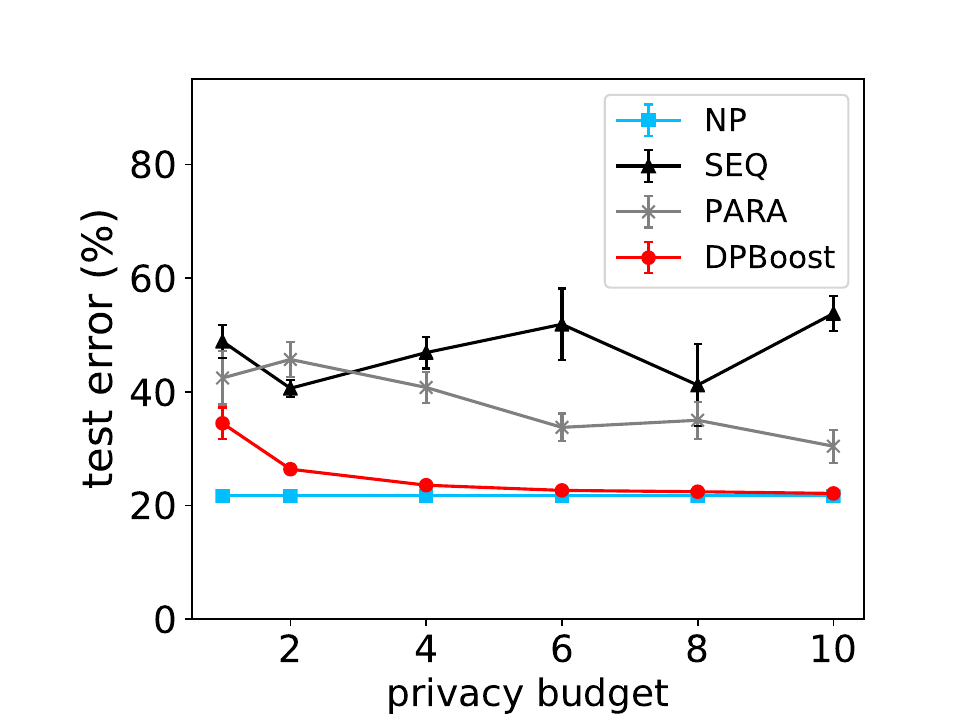}%
}
\subfloat[cod-rna]{\includegraphics[width=0.2\textwidth]{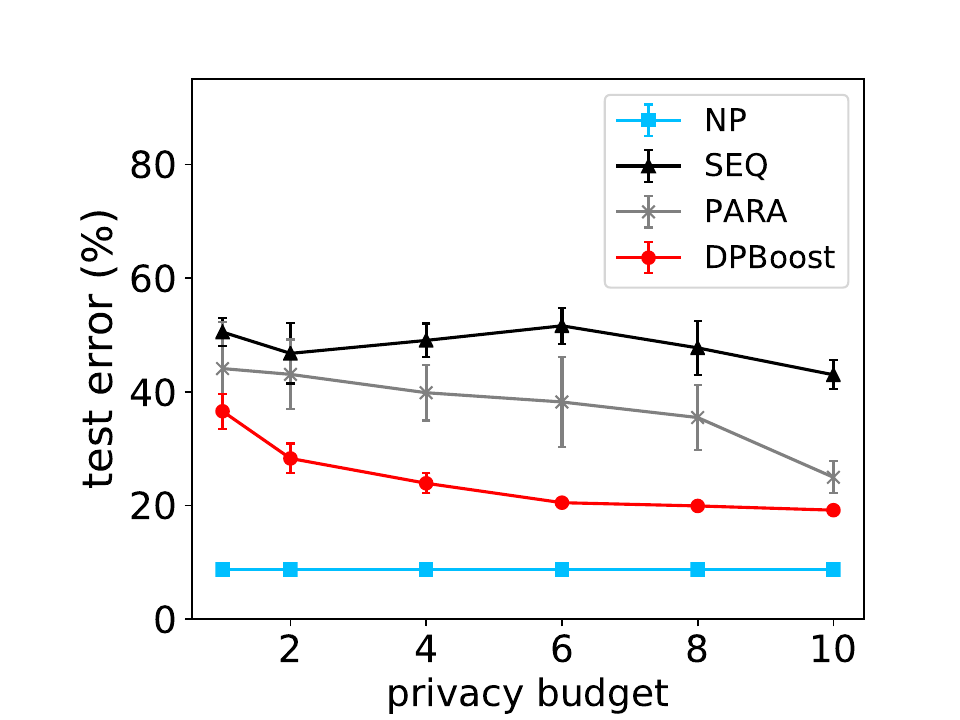}
}
\hfil
\subfloat[webdata]{\includegraphics[width=0.2\textwidth]{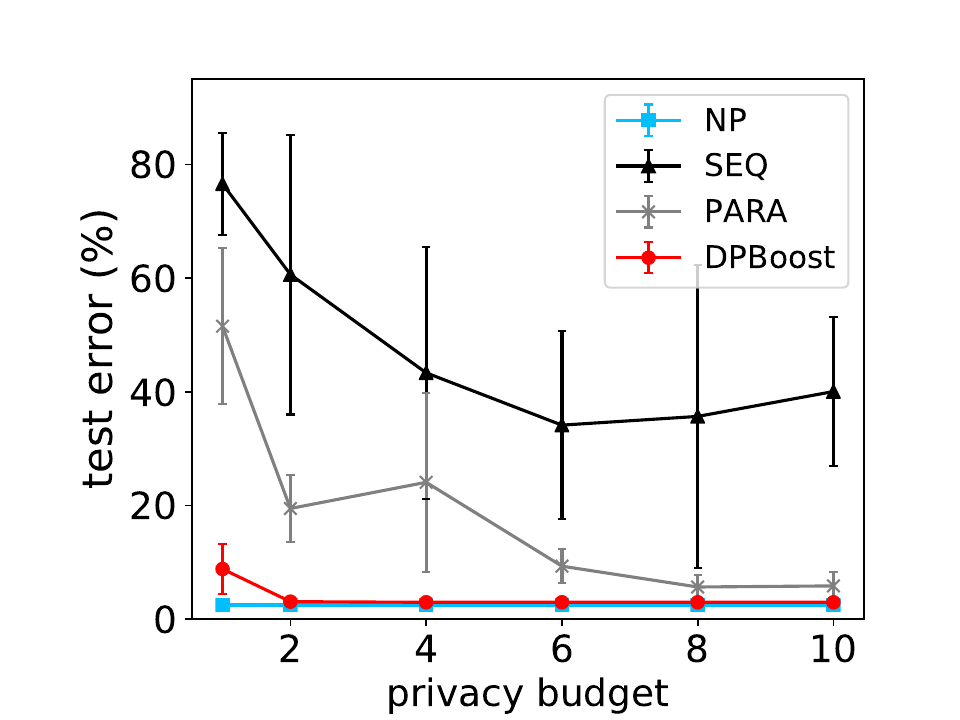}%
}
\subfloat[synthetic\_cls]{\includegraphics[width=0.2\textwidth]{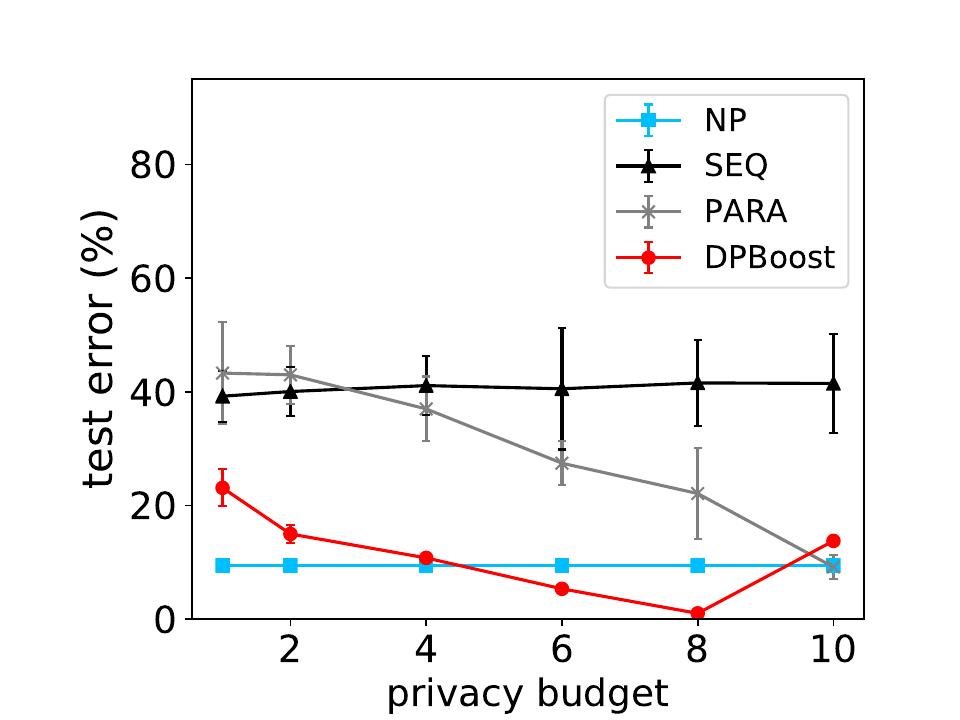}%
}
\subfloat[abalone]{\includegraphics[width=0.2\textwidth]{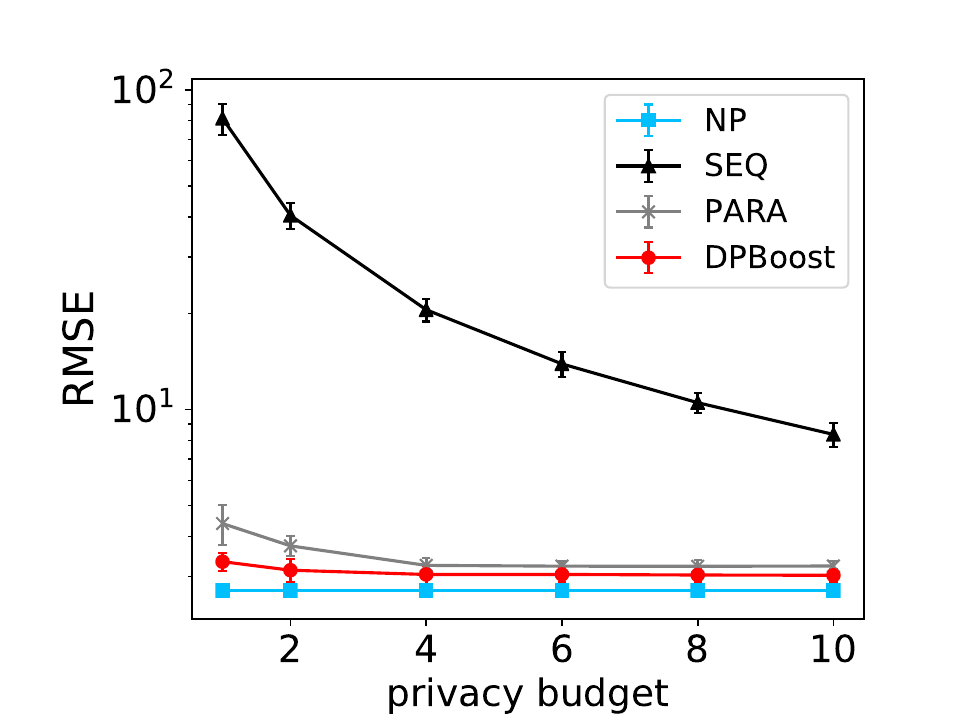}%
}
\subfloat[YearPredictionMSD]{\includegraphics[width=0.2\textwidth]{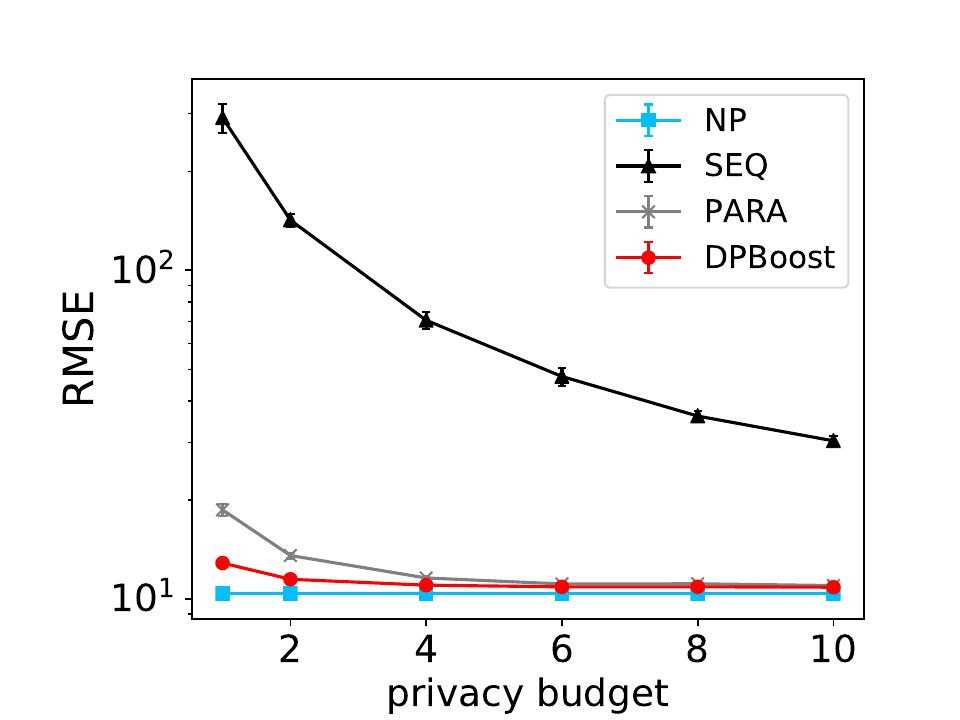}%
}
\subfloat[synthetic\_reg]{\includegraphics[width=0.2\textwidth]{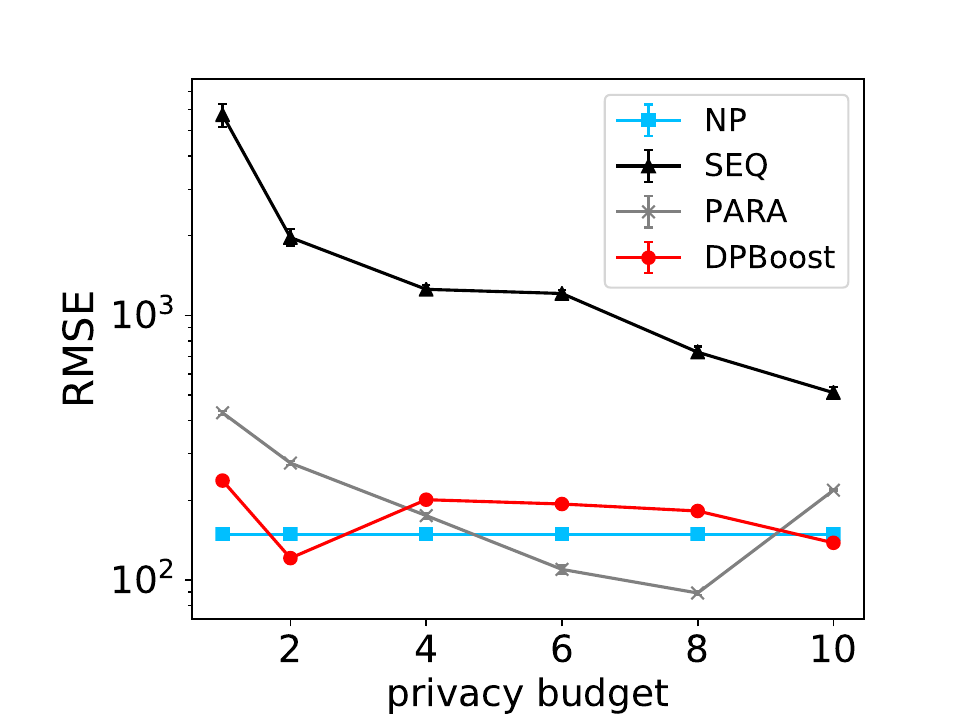}%
}
\caption{Comparison of the test errors/RMSE given different total privacy budgets. The number of trees and the number of trees inside an ensemble is set to 40.}
\label{fig:error_budget_40}
\end{figure*}

\begin{figure*}
\centering
\subfloat[adult]{\includegraphics[width=0.2\textwidth]{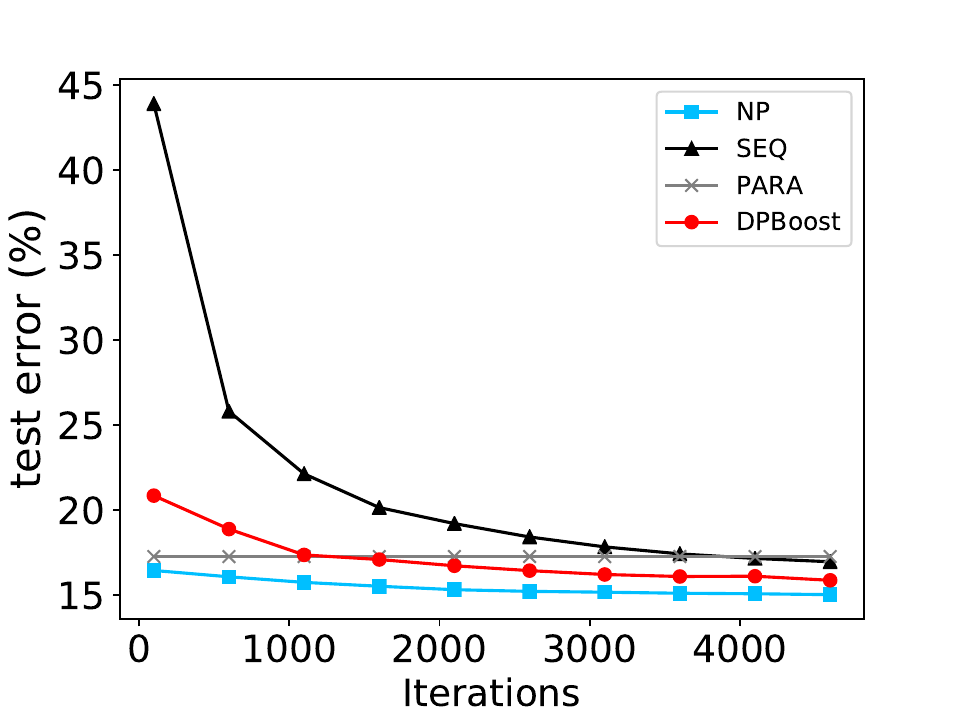}%
}
\subfloat[real-sim]{\includegraphics[width=0.2\textwidth]{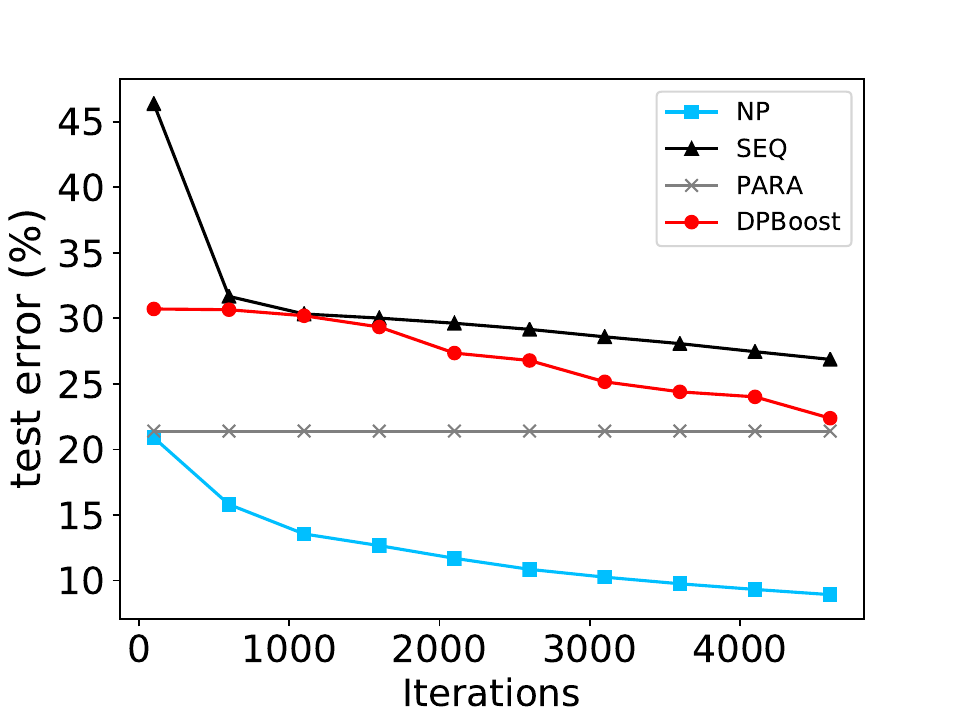}%
}
\subfloat[covtype]{\includegraphics[width=0.2\textwidth]{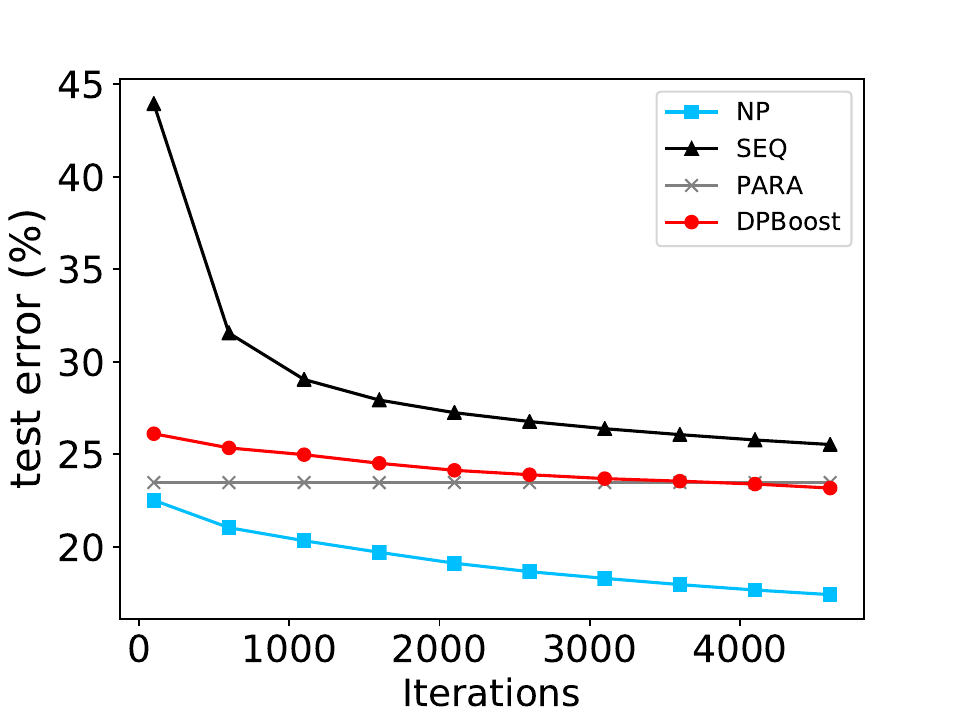}%
}
\subfloat[susy]{\includegraphics[width=0.2\textwidth]{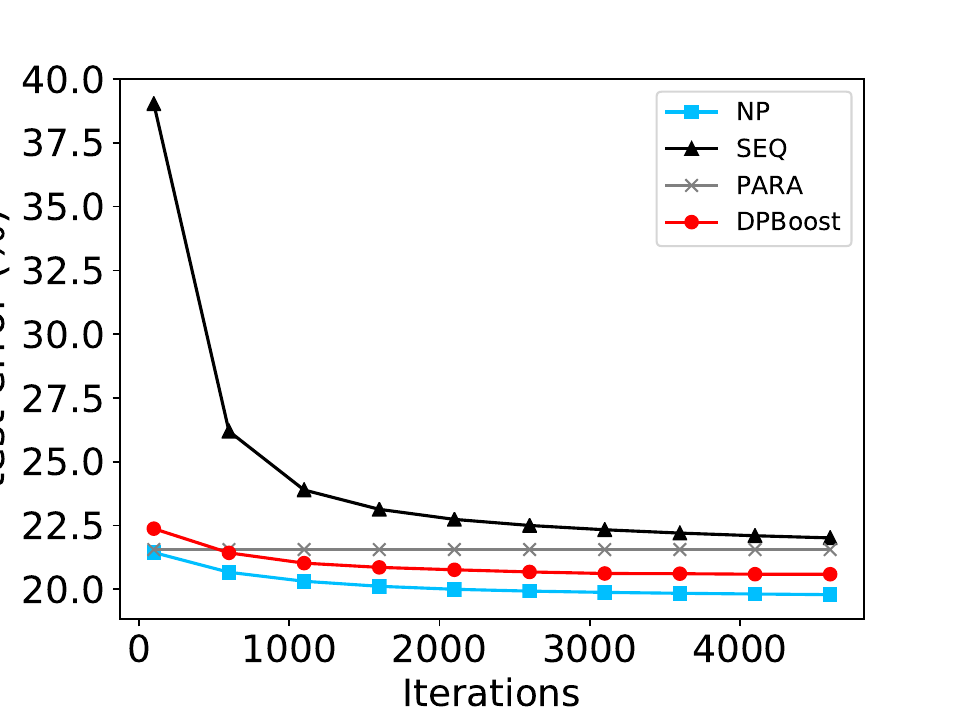}%
}
\subfloat[cod-rna]{\includegraphics[width=0.2\textwidth]{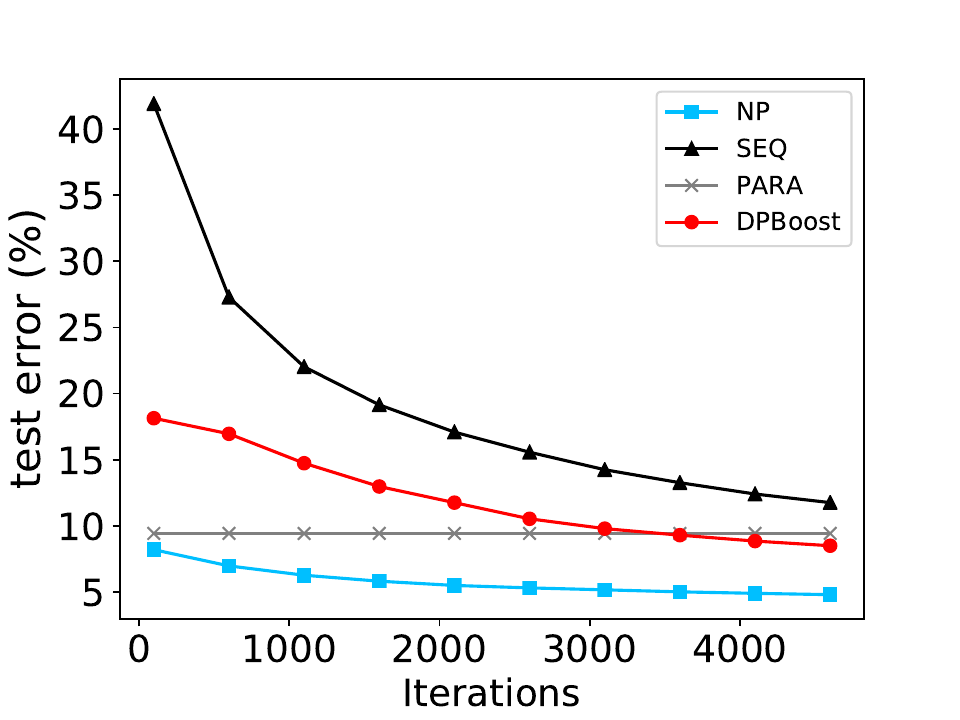}%
}
\caption{Comparison of test error convergence. The number of trees is set to 1000. The number of trees inside an ensemble is set to 20.}
\label{fig:convergence_20}
\end{figure*}

\begin{figure*}
\centering
\subfloat[adult]{\includegraphics[width=0.2\textwidth]{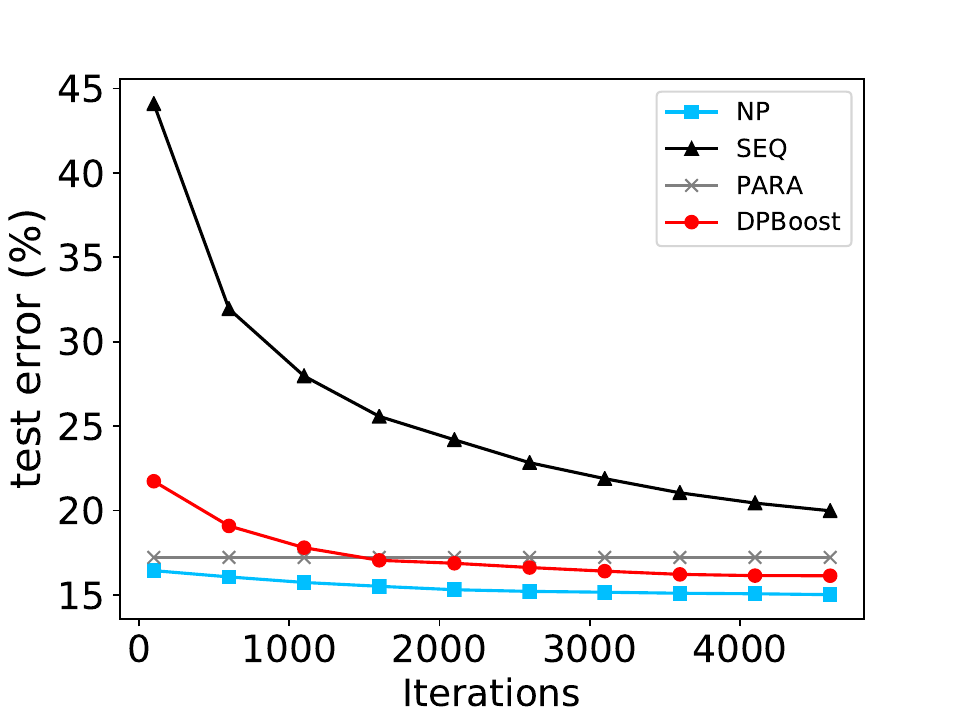}%
}
\subfloat[real-sim]{\includegraphics[width=0.2\textwidth]{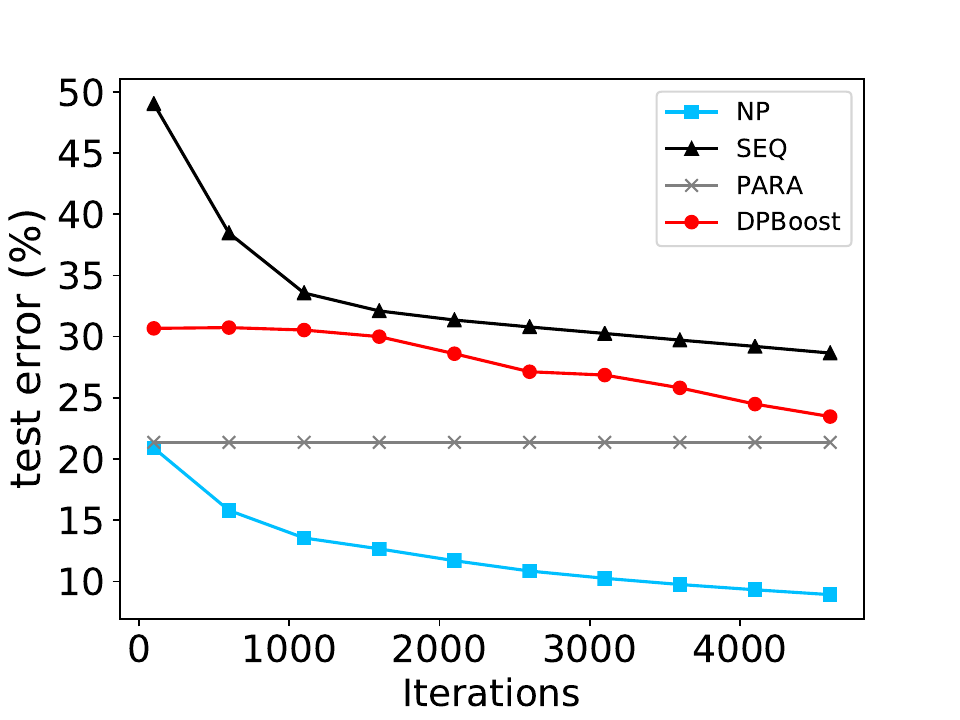}%
}
\subfloat[covtype]{\includegraphics[width=0.2\textwidth]{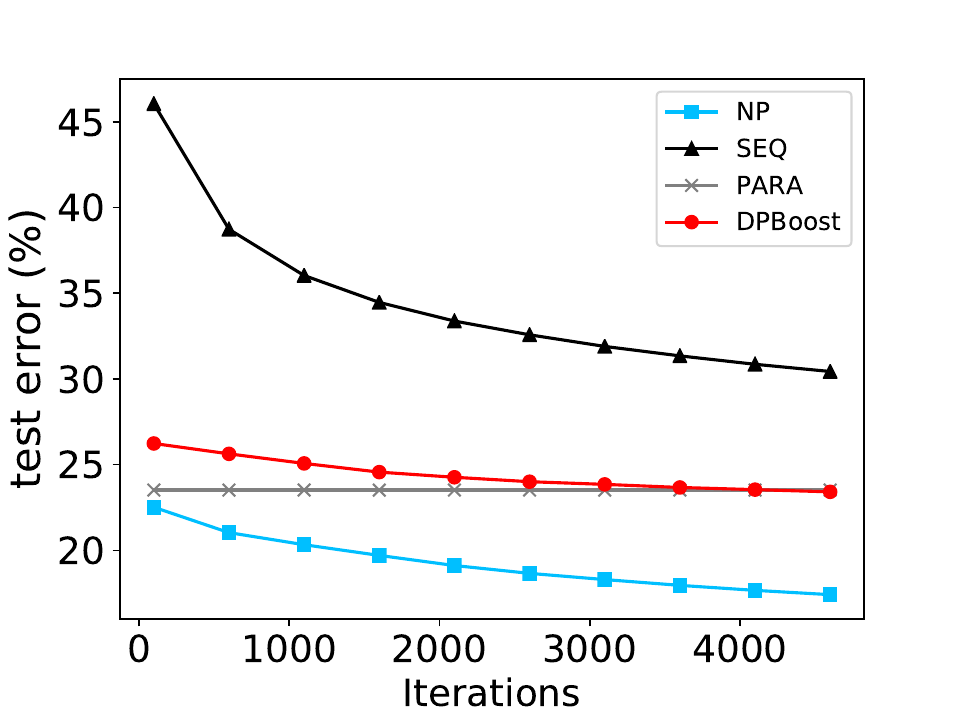}%
}
\subfloat[susy]{\includegraphics[width=0.2\textwidth]{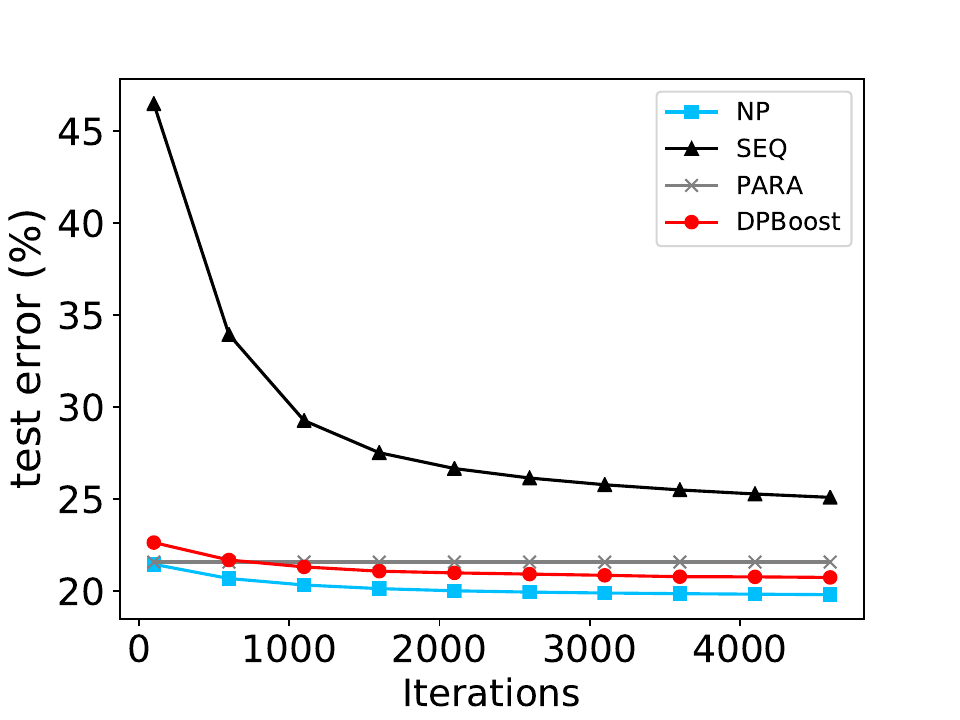}%
}
\subfloat[cod-rna]{\includegraphics[width=0.2\textwidth]{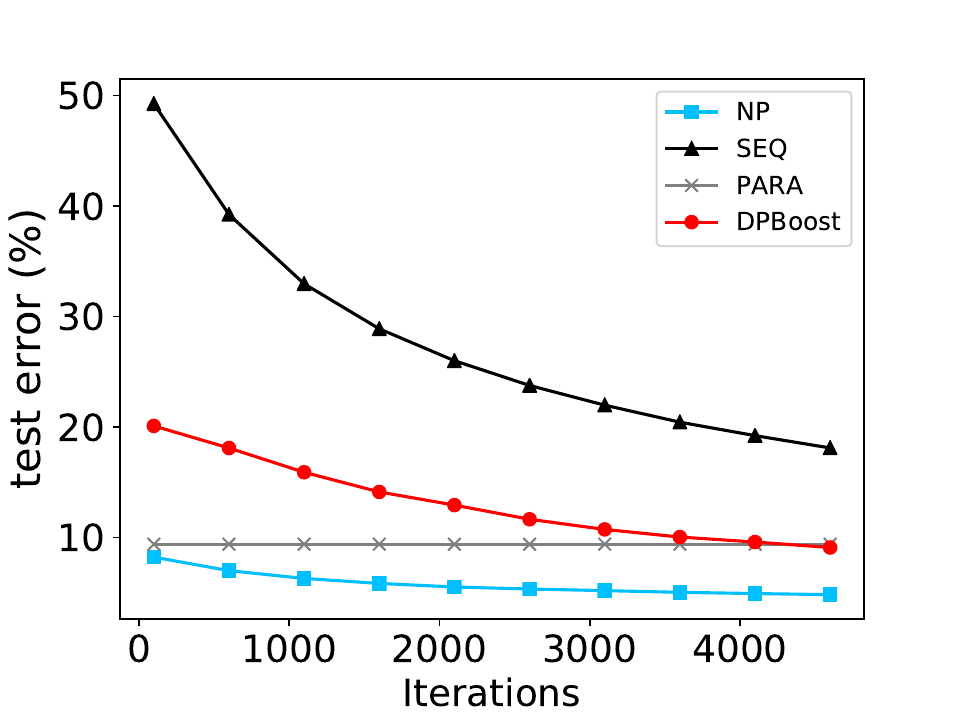}%
}
\caption{Comparison of test error convergence. The number of trees is set to 1000. The number of trees inside an ensemble is set to 40.}
\label{fig:convergence_40}
\end{figure*}





\end{document}